\documentclass[journal]{IEEEtran}

\usepackage{times}
\usepackage{epsfig}
\usepackage{graphicx}
\usepackage{amsmath}
\usepackage{amssymb}
\usepackage{bm}
\usepackage{subfigure}
\usepackage{caption}
\usepackage{enumerate}
\usepackage{threeparttable}
\usepackage{array}
\usepackage{ragged2e}
\renewcommand{\raggedright}{\leftskip=0pt \rightskip=0pt plus 0cm}
\newtheorem{theorem}{Theorem}
\newtheorem{proof}{Proof}

\ifCLASSINFOpdf
\else
\fi
\begin{document}
%
\title{A Systematic IoU-Related Method: Beyond Simplified Regression for Better Localization}
%
%
%

\author{Hanyang Peng,
        Shiqi Yu~\IEEEmembership{Member,~IEEE}
\thanks{Hanyang Peng and Shiqi Yu are with the Department of Computer Science and Engineering, Southern University of Science and Technology, Shenzhen,
China, 518055. Shiqi Yu is the corresponding author (email: yusq@sustech.edu.cn).} 
\thanks{The work is supported in part by the National Science Foundation of China (Grant No. 61806128, 61976144) and the National Key Research and Development Program of China (Grant No. 2020AAA0140002).}}

\maketitle

\begin{abstract}
Four-variable-independent-regression localization losses, such as Smooth-$\ell_1$ Loss, are used by default in modern detectors.
Nevertheless, this kind of loss is oversimplified so that it is  inconsistent with the final evaluation metric, intersection over union (IoU).  Directly employing the standard IoU is also not infeasible, since the constant-zero plateau in the case of non-overlapping boxes and the non-zero gradient at the minimum  may make it not trainable. Accordingly, we propose a systematic method to address these problems. Firstly, we propose a new metric, the extended IoU (EIoU),  which is well-defined when two boxes are not overlapping and reduced to the standard IoU when overlapping. Secondly, we present the convexification technique (CT) to construct a loss on the basis of  EIoU, which can guarantee the gradient at the minimum to be zero.  Thirdly, we propose a steady optimization technique (SOT) to make the fractional EIoU loss approaching the minimum more steadily and smoothly. Fourthly, to fully exploit the capability of the EIoU based loss,  we introduce an interrelated IoU-predicting head to further boost localization accuracy. With the proposed contributions, the new method incorporated into Faster R-CNN with ResNet50+FPN as the backbone yields  \textbf{4.2 mAP} gain on VOC2007 and \textbf{2.3 mAP} gain on COCO2017 over the baseline Smooth-$\ell_1$ Loss, at almost \textbf{no training and inferencing computational cost}. Specifically,  the stricter the metric is, the more notable the gain is, improving \textbf{8.2 mAP} on VOC2007 and \textbf{5.4 mAP} on COCO2017 at metric $AP_{90}$.
\end{abstract}


%
\IEEEpeerreviewmaketitle

\section{Introduction}

Object detection is a heavily-investigated topic in the computer vision community, because it is fundamental and the prerequisite for many other vision tasks, such as instance segmentation \cite{Mask_RCNN_2017, ins_seg_2016} and high-level object-based reasoning \cite{NIPS2017}.

\begin{figure}[t]
    \centering
    \includegraphics[width=1.2\linewidth]{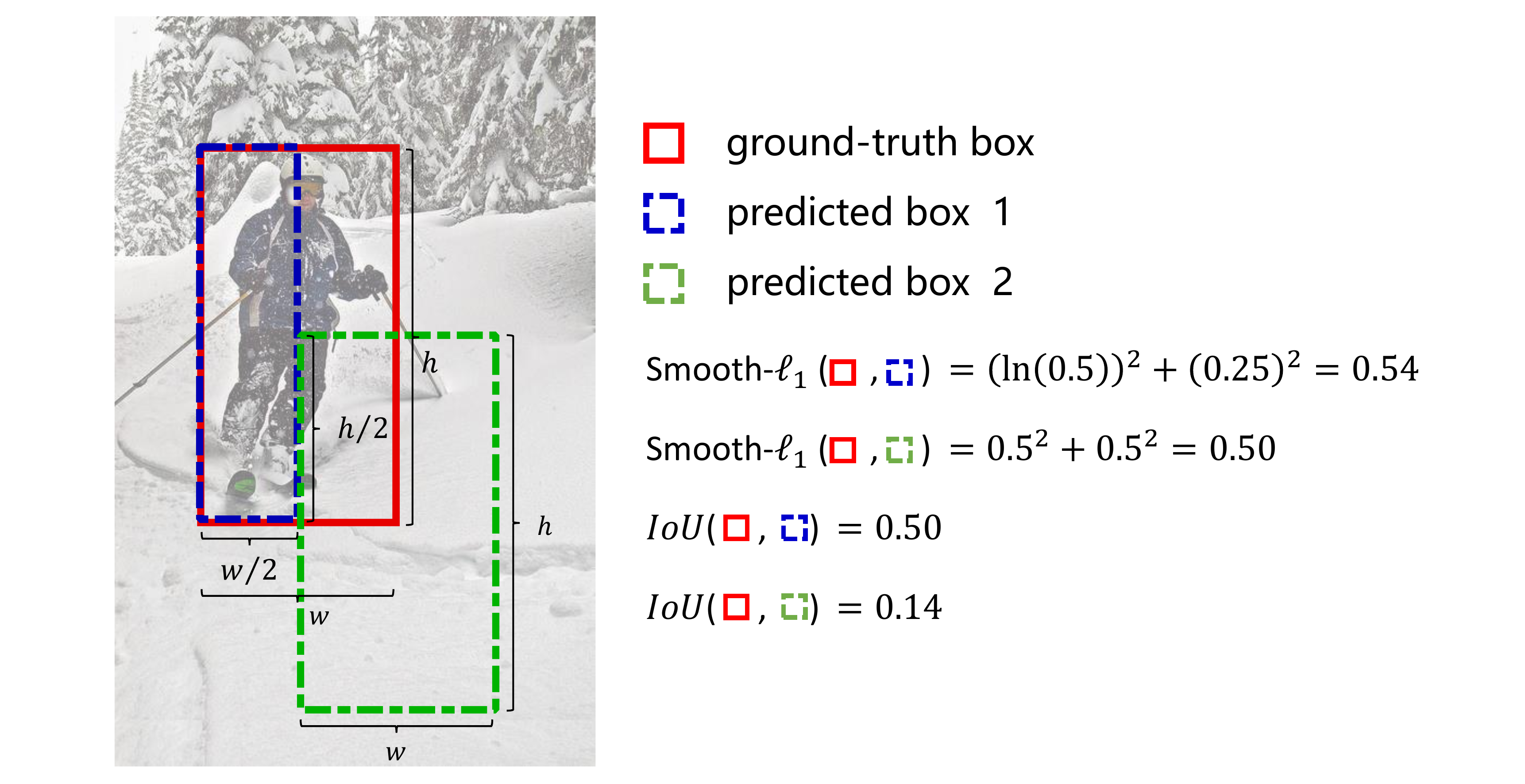}
    \caption{\small{ Illustration of the \emph{misalignment} between  Smooth-$\ell_1$ Loss and the metric $IoU$. Although  the predicted box 1 matches the ground-truth box better than the predicted 2 ($IoU$: 0.50 vs 0.14),    Smooth-$\ell_1$ Loss of the ground-truth box and the predicted box 1 is larger than that of the ground-truth box  and the predicted box 2 (Smooth-$\ell_1$ Loss: 0.54 vs 0.50).}}
    \label{Fig.1}
    \vspace{-1em}
\end{figure}

With the advent of deep CNNs \cite{Krizhevsky2012ImageNet, BatchNorm2015, ResNet2016} in  recent years, the performance of object detection has progressed substantially. There are generally two possible approaches to improve detection accuracy besides increasing samples: constructing ingenious architectures and devising better losses.
Constructing CNN architectures have made great strides in the past years \cite{Fast_RCNN_2015, Faster_RCNN_2017, FPN2017, Cascade_RCNN_2018, YOLOv3_2018, SSD_2016, Focal_loss_retinanet_2017, He2017Deep, Cheng1, Cheng2, Cheng3}.  One tendency  is  to design more and more sophisticated architectures for better performance, but this way commonly will increase the computational cost. In contrast, devising better losses is more economical,  since we can obtain the improvement at little cost of extra training and inferencing time. However, research on devising losses, especially localization losses, received much less attention in past years. Since R-CNN introduced a four-variable-independent-regression loss for localization in 2013~\cite{RCNN_2013}, the localization loss in modern deep detectors changed little. Although the four-variable-independent-regression loss is simple and straightforward,  it is not consistent with the final detection performance metric, $IoU$.   The gap between the four-variable-independent-regression loss and $IoU$ will inevitably result in some misaligned cases -- the loss is small, but $IoU$ is also small, which means the predicted box and the ground-truth box overlap little, and vice versa. An example in Fig 1 visually illustrates this gap between Smooth-$\ell_1$ Loss (the most widely-used four-variable-independent-regression loss) and $IoU$.  It is intuitive that equipping an $IoU$ related loss can address this problem. 

However,  the standard $IoU$ based losses did not popularize in the past years, since there are  two intrinsic deficiencies in the standard $IoU$. {(i)} When the predicted box and the ground-truth box do not overlap, the standard $IoU$  itself is ill-defined since the value is constant zero. Then, the gradient of any standard $IoU$ based loss will also become zero,  so the backpropagation cannot pull the predicted box close to the ground-truth in this non-overlapping case. {(ii)}  The gradient of a simple standard $IoU$ loss at the minimum where two boxes completely overlap is non-zero, which will bring about oscillation and slow convergence when applying gradient descent algorithms. Very recently, \cite{GIOU_2019} pioneeringly proposed  $GIoU$  that adds a regularization term after a standard $IoU$  loss, and then the new loss has non-zero gradients  when two boxes are not overlapping.  However,  the regularization term  also makes {$GIoU$} not equivalent to the standard $IoU$ when two boxes are overlapping. Hence the performance of {$GIoU$} might be suboptimal as the standard $IoU$ is the final evaluation metric. Moreover, $GIoU$ Loss  still does not overcome  problems of oscillation and slow convergence.  \cite{DIoU2020} presented $CIoU$ loss by incorporating the normalized distance between two boxes into the standard $IoU$. Actually, $CIoU$ can be considered as the combination of a four-variable-independent-regression loss and a standard $IoU$ loss. $CIoU$ converges much faster than $GIoU$, but it still can not avoid oscillation due to non-zero gradients at the minimum.

In this paper,  we will propose a systematic method to tackle all the problems above and introduce some new techniques to further improve the detection accuracy.
\begin{itemize}
    \item We  propose a more generalized and well-defined $IoU$, namely $EIoU$. In the case of overlapping bounding boxes, $EIoU$ is identical to the standard $IoU$, while in the case of non-overlapping boxes, $EIoU$ is smaller as two boxes separate further, which will make $EIoU$ trainable.
    \item We present a convexification technique (CT) to construct a new loss. It will lead the gradient to become zero at the minimum. So it is possible to achieve the minimum through gradient descent algorithms. Moreover, just like  Focal Loss, the convexification technique will adaptively assign higher weight on hard examples.
    \item We introduce a steady optimizing technique (SOT) to make the loss approach the minimum steadily and smoothly. The convergence of the steady optimization technique is  theoretically ensured.
    \item Harnessing the computed ground-truth $IoU$ score in the new loss above,  we add a single-layer head to be trained to predict this $IoU$ score. Then,  we can utilize the predicted IoU score to help non-maximum suppression(NMS) select more precise bounding boxes in the inferencing stage.
\end{itemize}

\section{Related Work}

\noindent\textbf{Architectures of CNN based detectors.} The architecture of modern CNN based detectors can be generally divided into two parts: the backbone network and the detection-specific network.  The backbones are commonly borrowed from the networks designed for categorization, of which VGG \cite{VGG_2015}, ResNet \cite{ResNet2016}, ResNeXt \cite{ResNeXt_2017}  are often leveraged. Besides, some specially-designed backbones for detection were also proposed in past years, such as DarkNet \cite{YOLO9000_2017} and DetNet \cite{DetNet_2018}, and Hourglass Net \cite{HourglassNet2016} are also frequently adopted.

There are two different logics to design a  detection-specific network. The first one is the two-stage network, and it consists of two sub-networks, where the first one is to generate a sparse set of candidate proposers, and the other is to determine the accurate location and categories based on the proposals. R-CNN  \cite{RCNN_2013}, fast R-CNN\cite{Fast_RCNN_2015} and Faster R-CNN  \cite{Faster_RCNN_2017} shaped the basic network architecture of two-stage detectors, and then R-FCN \cite{R-FCN_2016} replaced the fully-connected sub-network with a convolution sub-network to improve efficiency. FPN \cite{FPN2017} introduces a lateral network to produce object proposals at multiple scales with more contextual information.  Cascade R-CNN devised a cascade structure and it improves performance substantially \cite{Cascade_RCNN_2018}. \cite{IoUNet_2018} proposed IoU-Net and IoU guided NMS to acquire location confidence for accurate detection.  Grid R-CNN \cite{GridRCNN2019} can capture the spatial information explicitly and enjoys the position-sensitive property of fully convolutional architecture. Very recently, TridentNet\cite{TrideNet2019} constructed a parallel multi-branch architecture aiming to generate scale-specific feature maps with a uniform representational power.

Another one is the one-stage network, which directly predicts the locations and categories of the object instance. YOLO \cite{YOLO_2015} and SSD \cite{SSD_2016} first popularized the one-stage methods by much reducing the computational cost but still maintaining competitive performance. Then, DSSD \cite{DSSD_2017} and RON \cite{RON_2017}  introduced a network similar to the hourglass network to combine low-level and high-level information. RetinaNet~\cite{Focal_loss_retinanet_2017} with Focal loss as the one-stage detectors first outperformed the two-stage detectors. RefineDet \cite{RefineDet2018} designed the anchor refinement module and the object detection module to reduce negative boxes and improve detection.  CornerNet \cite{CornerNet_2018} is an anchor-free framework and adopts two subnetworks to detect the top-left and bottom-right key points and then employs a grouping subnetwork to pair them. Later some other competitive anchor-free detectors, such as FSAF \cite{FSAF2019}, FCOS \cite{FCOS2019} and CenterNet \cite{zhou2019objects, CenterNet2019},  were further developed.

These ingenious architectures significantly promoted the evolution of object detection. It is worth noting that the improvement of detection performance is partly attributed to  sophisticated  backbones and detection-specific networks that will commonly bring extra computational cost.

\noindent\textbf{Losses of CNN based detectors.} Compared with the design of architectures, the exploration of losses is more economical, because a well-devised loss can obtain performance gain with little additional train time cost and no extra test time cost. However, research on losses for detection has been underestimated for a long time.

Modern CNN based detectors were popularized by R-CNN in 2013~\cite{RCNN_2013},  and it introduced the softmax loss and a four-variable-independent-regression loss for classification and localization. Since then, this type of classification loss and localization loss  became mainstream and were applied to  most detectors. As for the classification loss,  YOLO~\cite{YOLO_2015} used to  employ the $\ell_2$ loss for categorization, but the later improved YOLO9000~\cite{YOLO9000_2017} gets back on track to reuse the softmax loss. Afterwards, Focal Loss ~\cite{Focal_loss_retinanet_2017} was specially developed to address extreme foreground-background ratio problem in one-stage detectors. It can adaptively down-weight overwhelming well-classified background examples to enjoy better detection performance. Recently, \cite{Cheng1} exploits  new losses to address the object rotation problem and the within-class diversity problem.

In terms of localization loss, Fast R-CNN substitutes the four-variable-independent-regression $\ell_2$ loss using in R-CNN with  Smooth-$\ell_1$ loss \cite{Fast_RCNN_2015}. The localization loss of the latter CNN based detectors mostly follow  Smooth-$ell_1$ loss with no or little change \cite{Faster_RCNN_2017, YOLO_2015, Focal_loss_retinanet_2017, FPN2017, IoUNet_2018}. However, as illustrated in Section 1 and Fig 1, there is misalignment  between Smooth-$\ell_1$ loss and the evaluation metric of $IoU$. So \cite{IoU_Loss_2016} tried to introduce a standard $IoU$ based loss to address this problem. Nevertheless, the standard $IoU$  also has its own defect. As long as two boxes are mutually detached no matter how far the distance is, the standard $IoU$  will become constant zero, so that the gradient of a standard $IoU$ based loss will also become zero and the loss is not trainable in this case.  $GIoU$ \cite{GIOU_2019} introduced a well-designed term added after a standard IoU based loss, and then the new loss becomes non-zero when two boxes are separated. This pioneering work made great progress to make $IoU$ based loss feasible. But just the adding term makes this new loss no longer equal the standard $IoU$. Hence it may lead to an  unexpected result that  $GIoU$ Loss in some cases of overlapping boxes is larger than that  in some cases of non-overlapping boxes.

In this work, we will propose a systematical method to  tackle the problems above of existing localization losses.

\section{The Proposed Approach}
In this section, we will present this systematic approach.  We first introduce the standard $IoU$, and interpret its plight for handling the situation that two boxes are non-overlapping. Next, we will show how we devise a new extended $IoU$ that can overcome the difficulty above. Then, we will use a convexification technique/focal technique to construct an extended $IoU$ based loss. Afterward, we will provide a steady optimization technique to make the training process steadily and smoothly. Finally, we will present an interrelated IoU-predicting head to select more precise predicted bounding boxes.

\subsection{ Standard IoU}

Constructing an $IoU$ based loss is an intuitive way to tackle the unappealing problems that the four-variable-independent-regression losses bring. However, the standard IoU ($SIoU$)  has some deficiencies that hinder the prevalence of $IoU$ based losses, and we will elaborate it in the following.

Given the targeted bounding boxes with a tuple $\left(x_1^t, y_1^t, x_2^t, y_2^t\right)$ and the predicted box with a tuple $\left(x_1^p, y_1^p, x_2^p, y_2^p \right)$, where $x_1$, $y_1$ and $x_2$, $y_2$ are the coordinate value of the top-left and bottom-right corners of the bounding boxes, respectively. When two boxes are overlapping,  the definition of the standard $SIoU$  is
 \begin{align}
&x_1 = \max\left(x_1^t,~ x_1^p \right), \label{Eq.1}\\
&y_1 = \max\left(y_1^t,~ y_1^p \right), \\
&x_2 = \min\left(x_2^t,~ x_2^p \right), \\
&y_2 = \min\left(y_2^t,~ y_2^p \right), \label{Eq.4}\\
&I_{\rm{std}} = (x_2 - x_1)(y_2 - y_1), \\
&S_t = (x_2^t - x_1^t)(y_2^t - y_1^t),  \label{Eq.6}\\
&S_p = (x_2^p - x_1^p)(y_2^p - y_1^p), \\
&U_{\rm{std}} = S_t + S_p - I_{\rm{std}},\label{Eq.8} \\
&SIoU = \frac{I_{\rm{std}}}{U_{\rm{std}}} \label{Eq.9}.
 \end{align}

However, when two boxes are not overlapping, the value of the intersection $I_{\rm{std}}$ and $SIoU$  is constant $0$,  which will bring two drawbacks. 
\begin{itemize}
\item  $SIoU$ cannot distinguish whether the two boxes are just in the vicinity or they are separated remotely. 
\item The gradient of the $SIoU$ for backpropagation  will also become zero. 
\end{itemize}
Hence $SIoU$  is not trainable in this case\footnote{ Actually, $SIoU$ is not trainable only when all the pair boxes are non-overlapping. In practice, it is common there are overlapping pair boxes and non-overlapping pair boxes in a batch. Hence the total gradient of a batch might not be zero. However, the exist of non-overlapping boxes in a batch will still make the performance for $SIoU$ degrade which can be seen in Table I. }.

\subsection{ Extended IoU }

In this subsection, we introduce our extended IoU ($EIoU$) that is accurately equivalent to the standard $IoU$ in the case of overlapping  boxes and has non-zero gradients in the case of non-overlapping  boxes.

 Conserving the definition of Eq.(\ref{Eq.1}-\ref{Eq.4}), the extended intersection ($I_e$) is

\begin{align}
&x_0 = \min\left(x_1^t, ~ x_1^p \right)  \label{Eq.11}\\
&y_0 = \min\left(y_1^t, ~ y_1^p \right) \\
&x_{\min} = \min\left(x_1, ~ x_2 \right) \\
&y_{\min} = \min\left(y_1, ~ y_2 \right) \\
&x_{\max} = \max\left(x_1, ~ x_2 \right) \\
&y_{\max} = \max\left(y_1, ~ y_2 \right)\label{Eq.15}
\end{align}
\begin{equation}
\begin{aligned}
 I_e = &S_1 + S_2 + S_3 + S_4 \\
 =&(x_2 - x_0)(y_2 - y_0) + (x_{\min} - x_0)(y_{\min} - y_0)  \\
&- (x_1 - x_0)(y_{\max} - y_0) - (x_{\max} - x_0)(y_1 - y_0), \\
\end{aligned}
\label{Eq.17}
\end{equation}
where we define  $I_{\rm{e}} = S_1 + S_2 - S_3 -S_4$, in which $S_1$ is  area of the rectangle with top-left corner $(x_0,y_0)$ and bottom-right $(x_2, y_2)$; $S_2$ is area of the rectangle with top-left corner $(x_0,y_0)$ and bottom-right $(x_{\min}, y_{\min})$; $S_3$ is area of the rectangle with top-left corner $(x_0,y_0)$ and bottom-right $(x_1,y_{\max})$; $S_4$ is area of the rectangle with top-left corner $(x_0,y_{0})$ and bottom-right $(x_{\max}, y_{1})$.

\begin{figure}[tbp]

    \begin{minipage}{0.45\linewidth}
    \centerline{\includegraphics[width=1\linewidth]{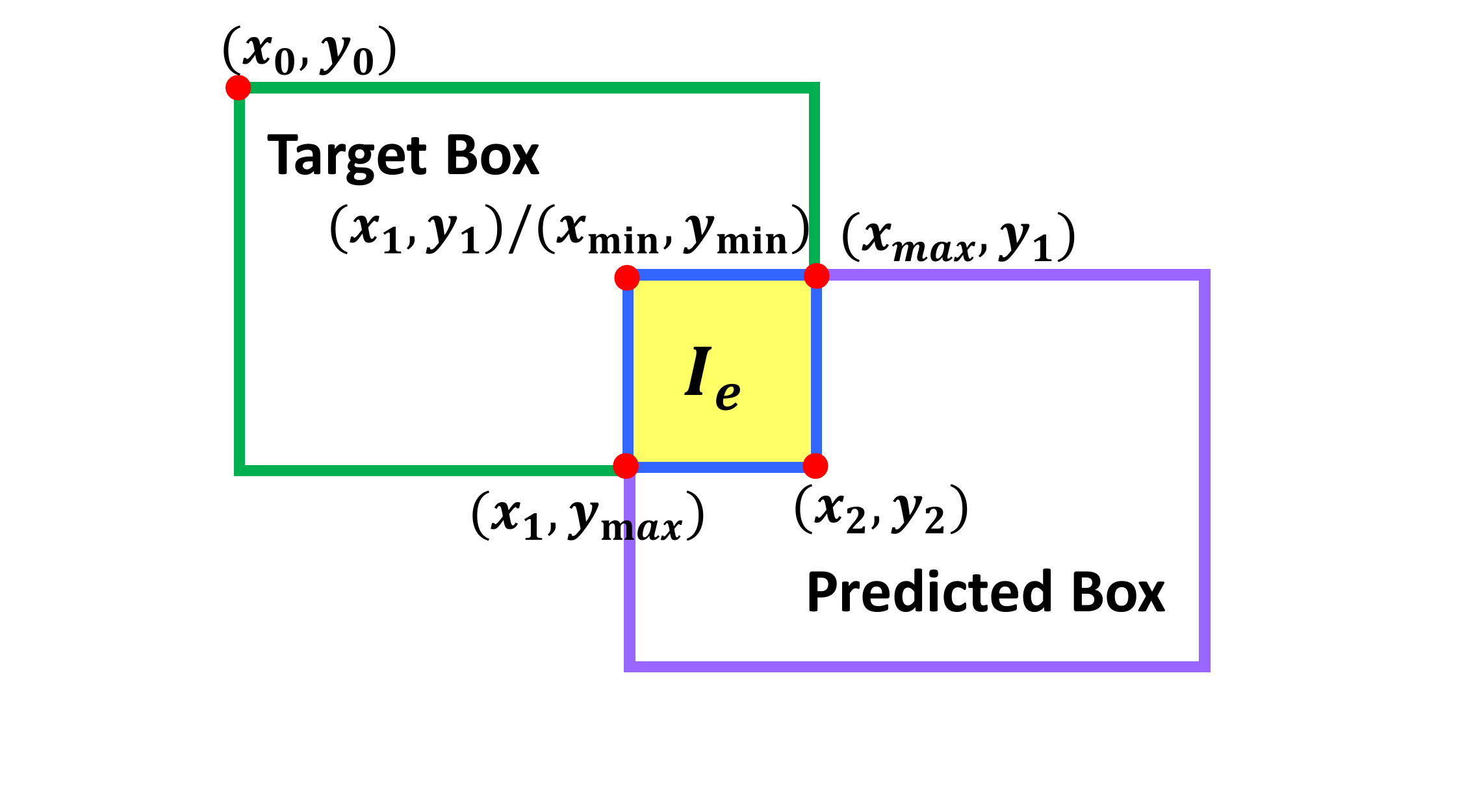}}
    \vspace{-10pt}
    \centerline{ \tiny{(a) Overlapping ($x_1 < x_2$, $y_1<y_2$)}}
    \centerline{ \tiny{$IoU=EIoU=\frac{1}{11}$}}
    \end{minipage}
    \begin{minipage}{0.45\linewidth}
    \centerline{\includegraphics[width=1\linewidth]{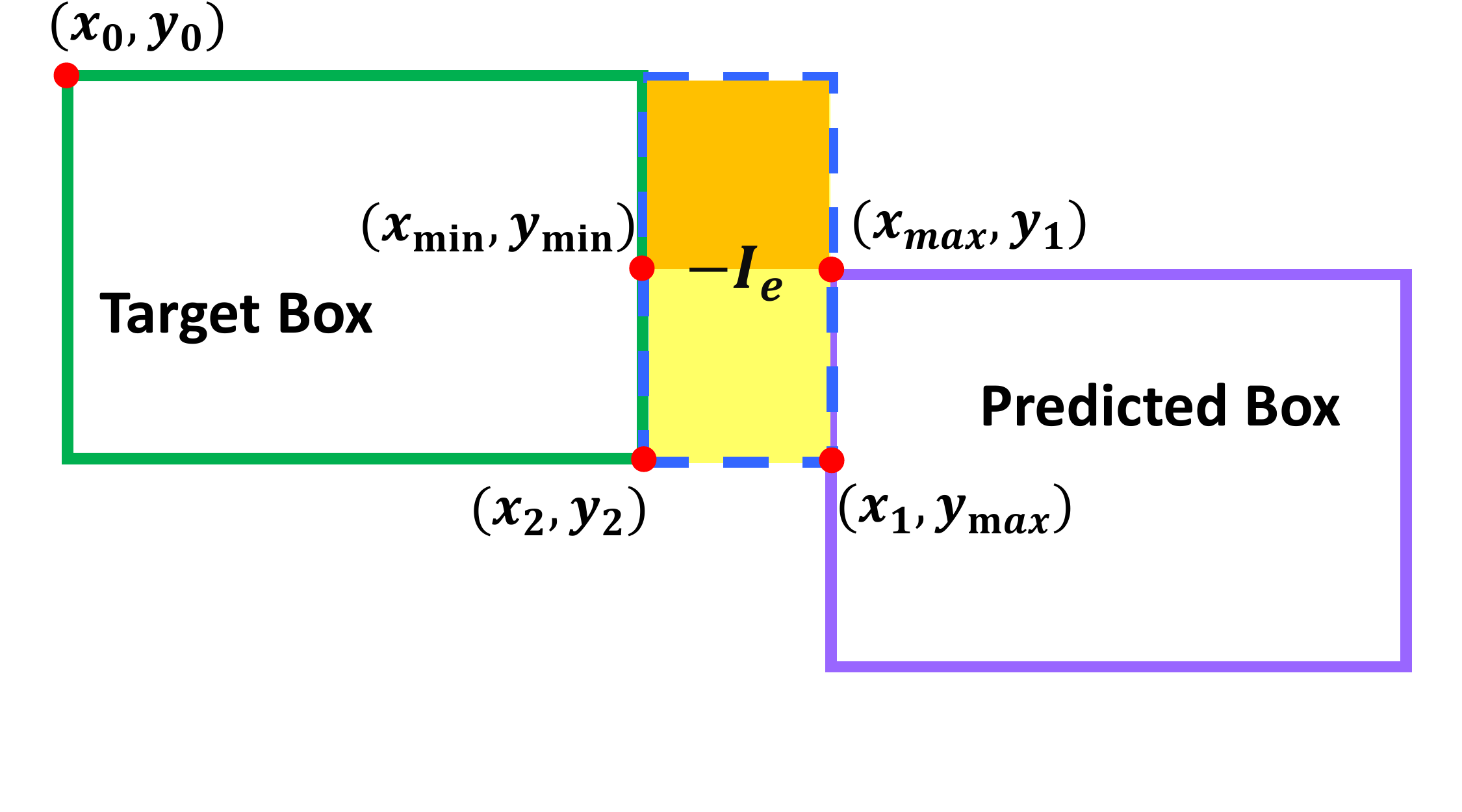}}
    \vspace{-10pt}
    \centerline{ \tiny{(b) Non-Overlapping ($x_1 > x_2$, $y_1<y_2$)}}
    \centerline{ \tiny{$IoU=0$, $EIoU=-\frac{1}{5}$}}
    \end{minipage}
    \\ \vspace{10pt}

    \begin{minipage}{0.45\linewidth}
    \centerline{\includegraphics[width=1\linewidth]{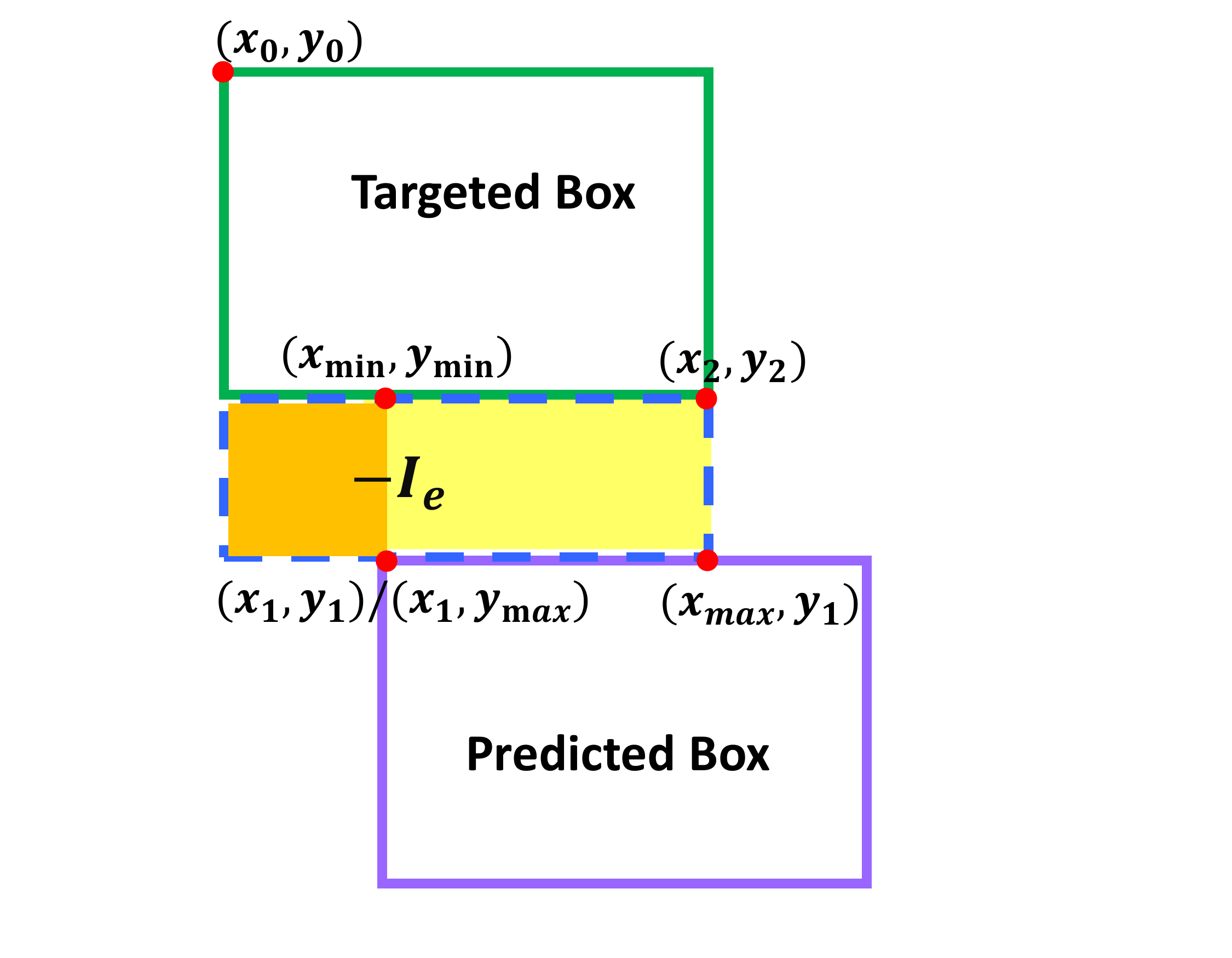}}
    \vspace{-10pt}
    \centerline{ \tiny{(c) Non-Overlapping ($x_1 < x_2$, $y_1>y_2$)}}
    \centerline{ \tiny{$IoU=0$, $EIoU=-\frac{1}{4}$}}
    \end{minipage}
    \begin{minipage}{0.45\linewidth}
    \centerline{\includegraphics[width=1\linewidth]{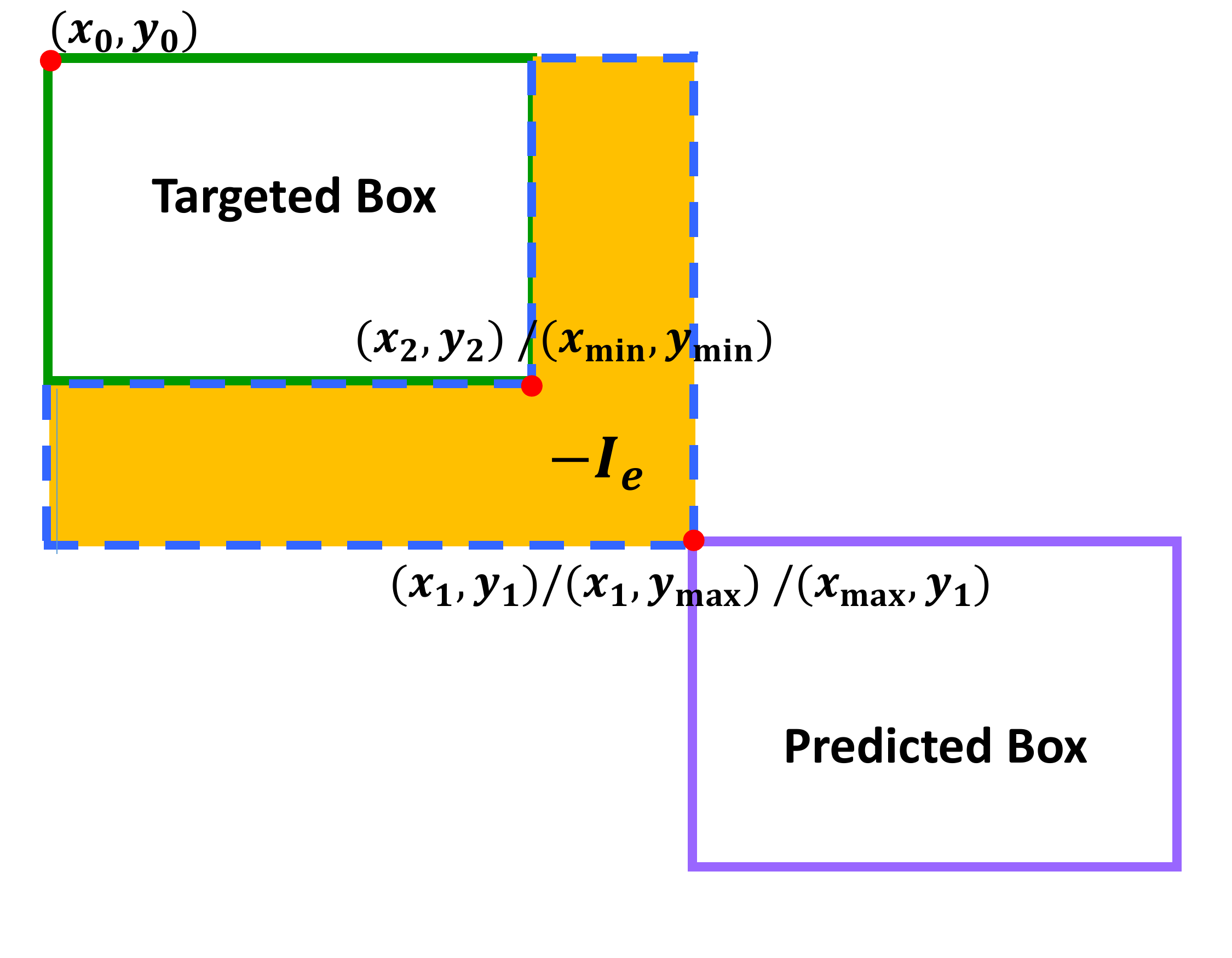}}
    \vspace{-10pt}
    \centerline{ \tiny{(d) Non-Overlapping ($x_1 > x_2$, $y_1>y_2$)}}
    \centerline{ \tiny{$IoU=0$, $EIoU=-\frac{5}{11}$}}
    \end{minipage}

    \caption{\small{ Illustration the difference between $EIoU$ and $SIoU$. It is known $IoU = \frac{I}{S_t + S_p - I} $ and $S_t$ and $S_p$ are fixed, so the differences between $I_e$ and $I_{std}$ are the key.   From Eq.(\ref{Eq.11}-\ref{Eq.17}), we know $I_{\rm{e}} = S_1 + S_2 - S_3 -S_4$, where $S_1$ is  area of the rectangle with top-left corner $(x_0,y_0)$ and bottom-right $(x_2, y_2)$; $S_2$ is area of the rectangle with top-left corner $(x_0,y_0)$ and bottom-right $(x_{\min}, y_{\min})$; $S_3$ is area of the rectangle with top-left corner $(x_0,y_0)$ and bottom-right $(x_1,y_{\max})$; $S_4$ is area of the rectangle with top-left corner $(x_0,y_{0})$ and bottom-right $(x_{\max}, y_{1})$. It is known the standard $I_{\rm{std}} = S_0$ , where $S_0$ is the area of the rectangle with top-left corner $(x_1,y_{1})$ and bottom-right $(x_{2}, y_{2})$. Thus,  when two boxes are overlapping as shown in (a) with $x_1 < x_2$ and $y_1<y_2$,  $I_{\rm{e}}$ is always positive and exactly equivalent to the standard  $I_{\rm{std}}$.  When two boxes are not overlapping, there are three situations shown in (b) with $x_1 > x_2$ and $y_1<y_2$, (c) with $x_1 < x_2$ and $y_1>y_2$ and (d) with $x_1 > x_2$ and $y_1>y_2$. In this case,  $I_{\rm{e}}$ become \emph{negative}. Moreover, unlike $I_{\rm{std}}$ that keeps \emph{constant $0$},  the further two boxes are mutually separated, the smaller the value of $I_{\rm{e}}$  is, which conforms to human's intuition better and make gradients of $I_{\rm{e}}$ \emph{non-zero}. Note that light yellow regions for $I_{\rm{e}}$ in (a), (b) and (c) are one-fold areas and deep yellow region in (b), (c) and (d) are two-fold areas. } }
    \label{Fig.2}
    \vspace{-12pt}
\end{figure}

We enumerate all the four situations whether two boxes are overlapping or not overlapping for the proposed $I_{\rm e}$ in the following.

(i) As shown in \emph{Fig \ref{Fig.2}(a)}, when two boxes are overlapping with $x_1 < x_2$ and $y_1 < y_2$ ,  we have $x_{\min} = x_1$, $x_{\max} = x_2$, $y_{\min} = y_1$ and $y_{\max} = y_2$, and then
\begin{equation}
\begin{aligned}
I_e = &(x_{\max} - x_0)(y_{\max} - y_0) + (x_{\min} - x_0)(y_{\min} - y_0)  \\
&- (x_{\min} - x_0)(y_{\max} - y_0) - (x_{\max} - x_0)(y_{\min} - y_0) \\
=&(x_{\max} - x_{\min})(y_{\max} - y_{\min}) \\
=&(x_{2} - x_{1})(y_{2} - y_{1}) \\
>& 0.
\end{aligned}
\label{Eq.19}
\end{equation}

(ii)  As shown in \emph{Fig \ref{Fig.2}(b)}, when two boxes are non-overlapping with $x_1 > x_2$ and $y_1 < y_2$,  we have $x_{\min} = x_2$, $x_{\max} = x_1$, $y_{\min} = y_1$ and $y_{\max} = y_2$, and then
\begin{equation}
\begin{aligned}
 I_e = &(x_{\min} - x_0)(y_{\max} - y_0) + (x_{\min} - x_0)(y_{\min} - y_0)  \\
&- (x_{\max} - x_0)(y_{\max} - y_0) - (x_{\max} - x_0)(y_{\min} - y_0) \\
=&(x_{\min} - x_{\max})(y_{\max} - y_{0}) + (x_{\min} - x_{\max})(y_{\min} - y_{0})\\
<& 0.
\end{aligned}
\end{equation}

(iii)  As shown in \emph{Fig \ref{Fig.2}(c)}, when two boxes are non-overlapping with $x_1 < x_2$ and $y_1 > y_2$,  we have $x_{\min} = x_1$, $x_{\max} = x_2$, $y_{\min} = y_2$ and $y_{\max} = y_1$, and then
\begin{equation}
\begin{aligned}
 I_e = &(x_{\max} - x_0)(y_{\min} - y_0) + (x_{\min} - x_0)(y_{\min} - y_0)  \\
&- (x_{\min} - x_0)(y_{\max} - y_0) - (x_{\max} - x_0)(y_{\max} - y_0) \\
=&(x_{\max} - x_{0})(y_{\min} - y_{\max}) + (x_{\min} - x_{0})(y_{\min} - y_{\max})\\
<& 0.
\end{aligned}
\end{equation}

(iv)  As shown in \emph{Fig \ref{Fig.2}(d)}, when two boxes are non-overlapping with $x_1 > x_2$ and $y_1 > y_2$,  we have $x_{\min} = x_2$, $x_{\max} = x_1$, $y_{\min} = y_2$ and $y_{\max} = y_1$, and then
\begin{equation}
\begin{aligned}
 I_e = &(x_{\min} - x_0)(y_{\min} - y_0) + (x_{\min} - x_0)(y_{\min} - y_0)  \\
&- (x_{\max} - x_0)(y_{\max} - y_0) - (x_{\max} - x_0)(y_{\max} - y_0) \\
=&2\left((x_{\min} - x_{0})(y_{\min} - y_0) - (x_{\max} - x_{0})(y_{\max} - y_0)\right)\\
<& 0.
\end{aligned}
\label{Eq.20}
\end{equation}

Therefore, $I_e$ is positive and reduced to $I_{\rm std}$ in the case of overlapping and $I_e$ is negative and decreases with the distance of two boxes in the case of non-overlapping.

Analogous to {$SIoU$} in Eq.(\ref{Eq.6}-\ref{Eq.9}), we obtain  $EIoU$ based on  $I_{\rm e}$ in Eq. (\ref{Eq.17}):
\begin{align}
 &U_e = S_t + S_p - I_e,      \label{Eq.21}           \\
 &EIoU = \frac{I_e}{U_e},
 \label{Eq.21_1}
\end{align}

$U_{\rm e}$ is a function of $I_{\rm e}$ and always larger than zero, so characteristics of $EIoU$ are similar to that of $I_{\rm e}$, which  are summarized as follows:
\begin{itemize}
  \item When two boxes are attached,  $EIoU$  \emph{exactly equivalent} to the standard ${IoU}$ and always larger than zero.
  \item When the two boxes are detached, $EIoU$ is smaller than zero and decreases with  distance of  two boxes, so that gradient descent algorithms can be employed to train the predicted box to approach the targeted box until matched.
\end{itemize}

\noindent \textbf{Differences From GIoU.} Both  $GIoU$ \cite{GIOU_2019} and the proposed $EIoU$ aim to address the problem of zero gradients when two boxes do not overlap, but there are still some significant distinctions between them.
As shown in Algorithm 1,  $GIoU$  adds an extra term after  ${SIoU}$, which can be considered as a regularization metric. The new term indeed makes \emph{GIoU} have non-zero gradients when two boxes are detached, but it also leads $GIoU$ to be not equivalent to ${SIoU}$ any more when two boxes are attached. This change will cause new problems. \emph{First},  it brings some counter-intuitive and unreasonable cases, and one example is visually illustrated in Fig \ref{Fig.3}. \emph{Second}, the performance of $GIoU$ might be suboptimal as \emph{SIoU} is the final evaluation metric. As for $EIoU$ is not a regularization method and an incremental modification of $GIoU$. We fundamentally address the root of the problem by redefining ${IoU}$, so that it is trainable in the case of non-overlapping and  equivalent(reduced) to  ${SIoU}$ in the case of overlapping. Accordingly, $EIoU$ will never encounter similar plights shown in Fig \ref{Fig.3}.

\renewcommand\arraystretch{1.0}
\begin{table}[!hptb]
\small
\begin{threeparttable}
\centering
    \begin{tabular}{p{8cm}l}
        \hline
        \textbf{Algorithm 1}: $GIoU$ in \cite{GIOU_2019} \\
        \hline
       \textbf{Input:} Two arbitrary bounding boxes: $A$ and $B$  \\
        \textbf{Output:} $GIoU$  \\
         \textbf{1.}  Find the smallest bounding box $C$ that encloses $A$ and $B$\\
         \textbf{2.} Compute the standard IoU: $SIoU = \frac{A\bigcap B}{A\bigcup B}$ \\
         \textbf{3.} Compute \emph{GIoU}:  $GIoU = SIoU -\frac{C \setminus \left(A\bigcup B\right)}{C}$ \\
        \hline
    \end{tabular}
\end{threeparttable}
\end{table}

\begin{figure}[htbp]
    \centering
    \vspace{-10pt}
    \subfigure[]{\includegraphics[width=0.5\linewidth]{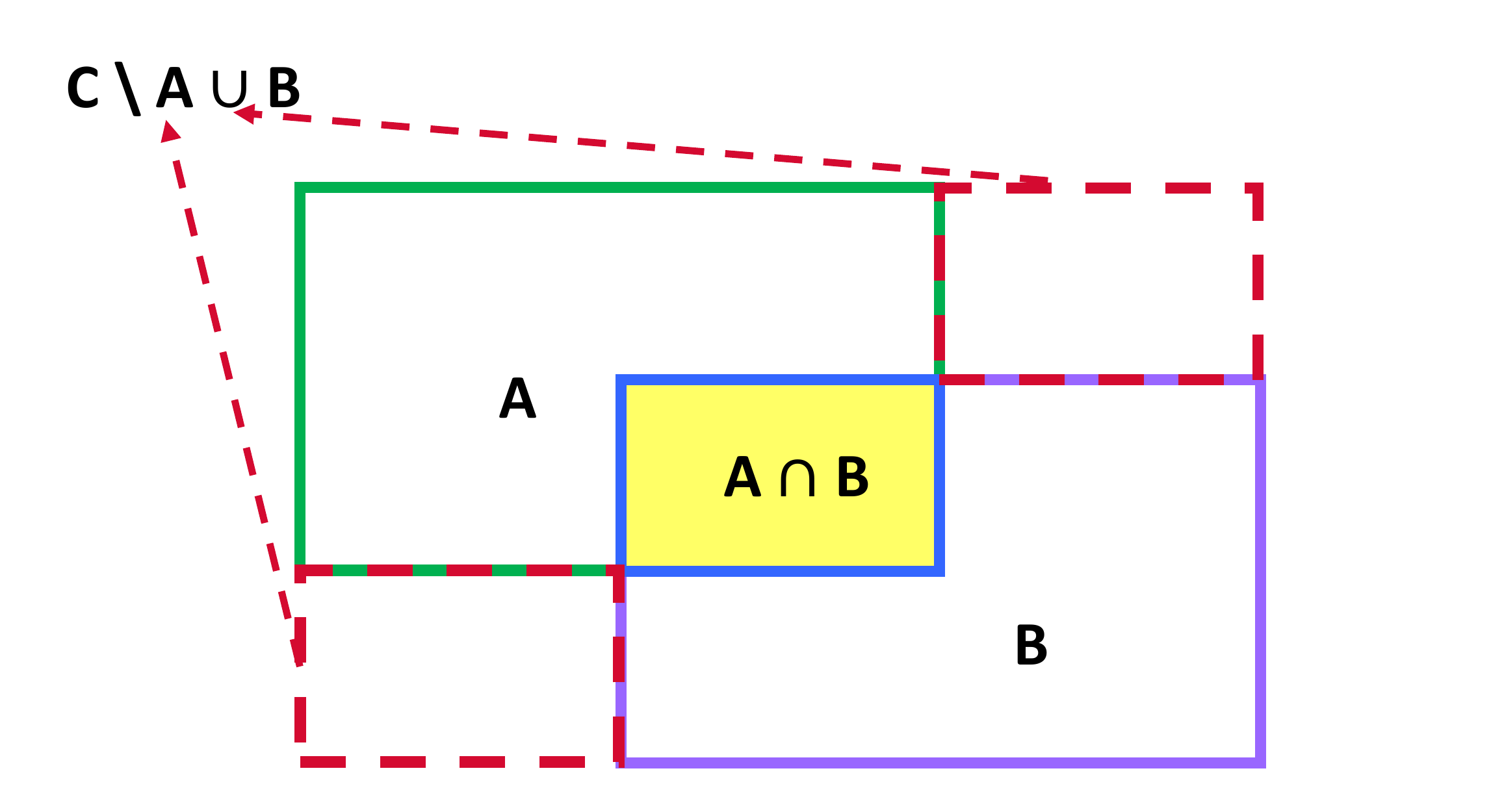}}\hspace{-5pt}
    \subfigure[]{\includegraphics[width=0.5\linewidth]{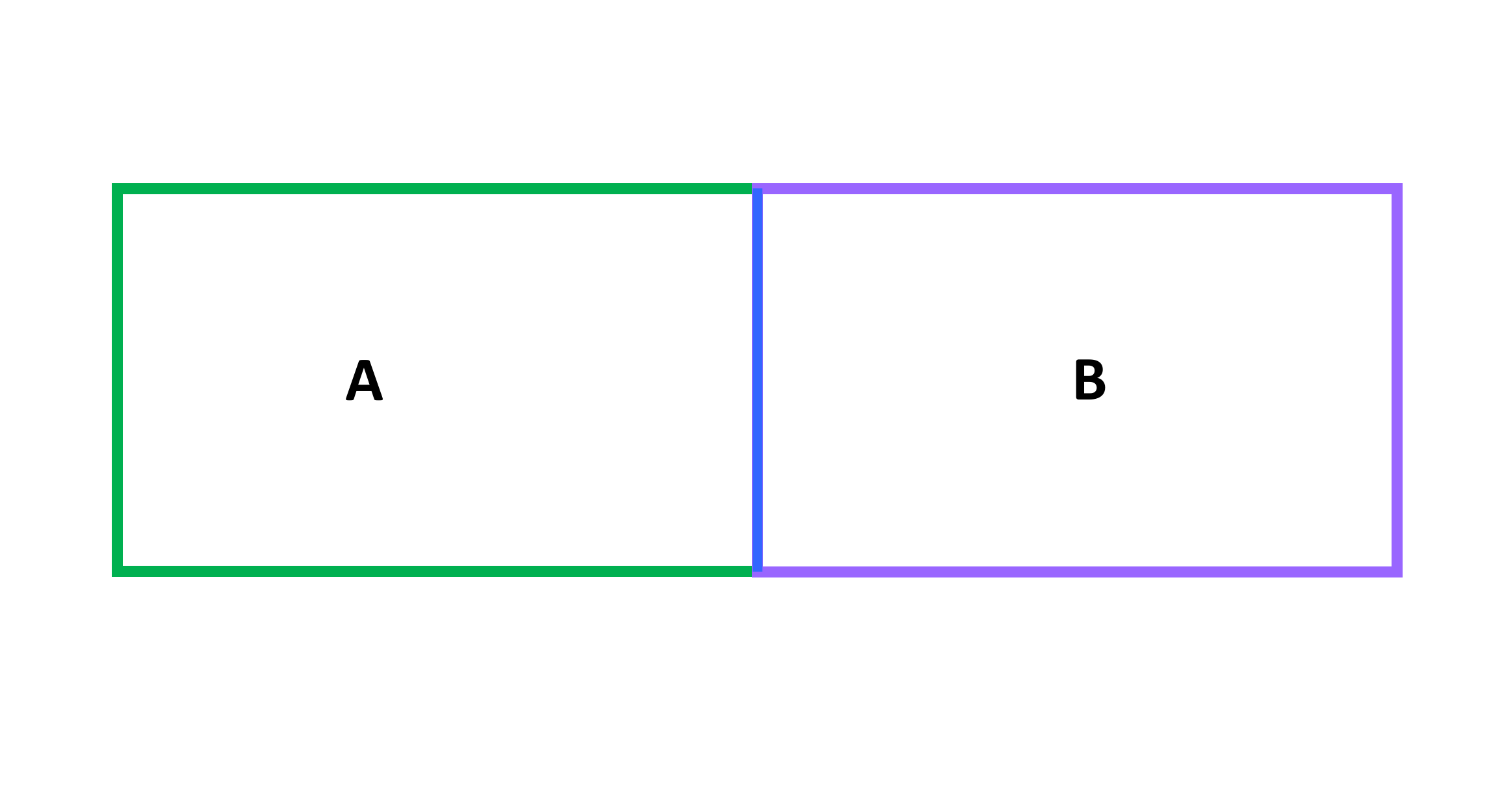}}
    \hspace{-10pt}
    {\caption{\small{ A \emph {counter-intuitive} case for $GIoU$. (a) We first compute the standard IoU: $SIoU = \frac{A\bigcap B}{A\bigcup B}=\frac{0.25}{2-0.25} = \frac{1}{7}$, and then compute $GIoU = SIoU -\frac{C \setminus \left(A\bigcup B\right)}{C}=\frac{1}{7} - \frac{0.5}{2.25} = -\frac{5}{63} $, while (b) two boxes are just attached, $ GIoU = SIoU = 0$. The value of $GIoU$ in (a) is  \emph{smaller} than that in (b), which is inconsistent with the fact that two boxes match better in (a) than in (b).}}
    \label{Figure.3}}
\end{figure}

\subsection{Covexification Technique (CT)}

\begin{figure*}[ht]
    \centering
    \subfigure[$- EIoU$ Loss]{\includegraphics[width=0.3\linewidth]{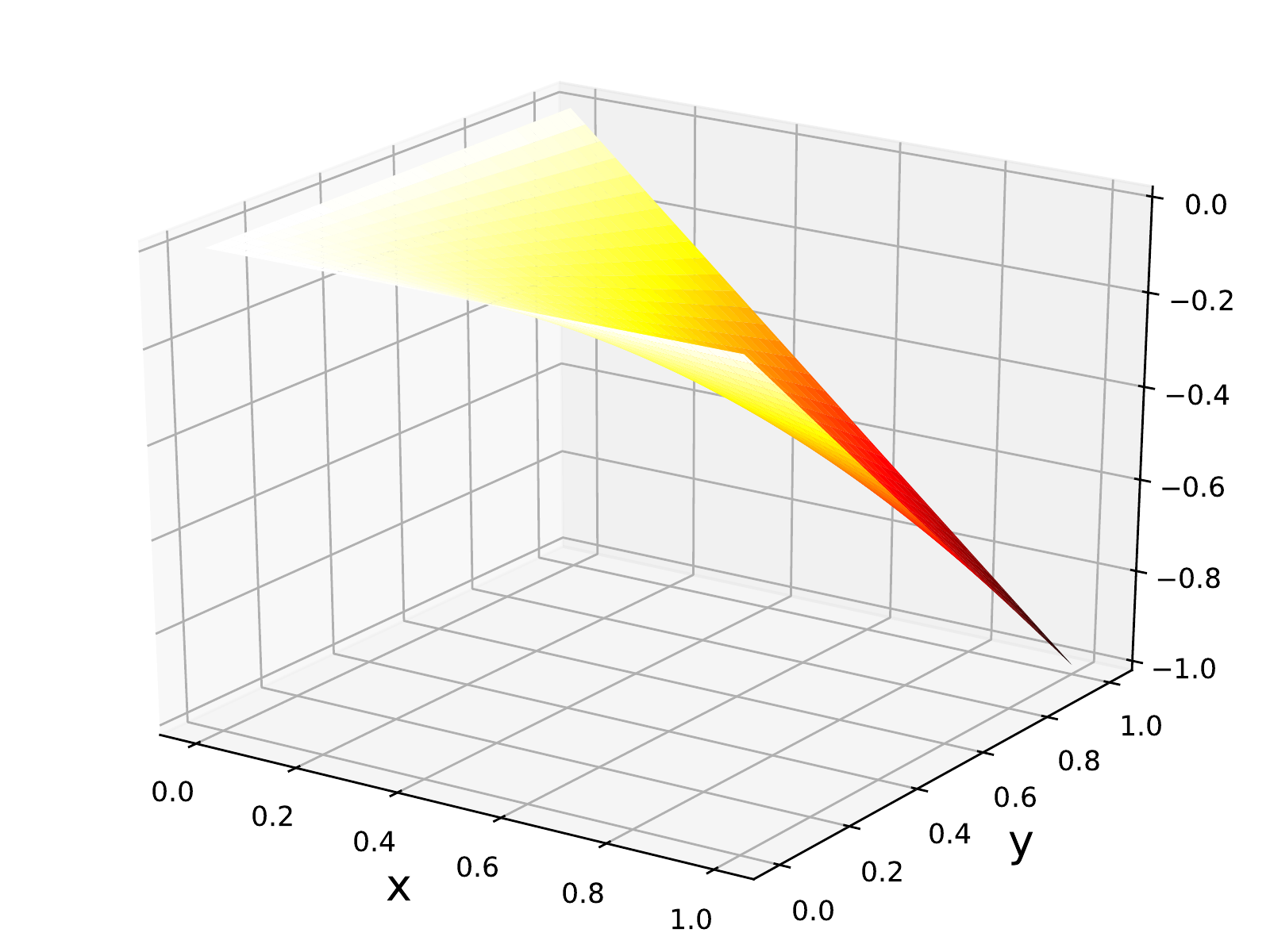}}
    \subfigure[Smooth-$EIoU$ Loss]{\includegraphics[width=0.3\linewidth]{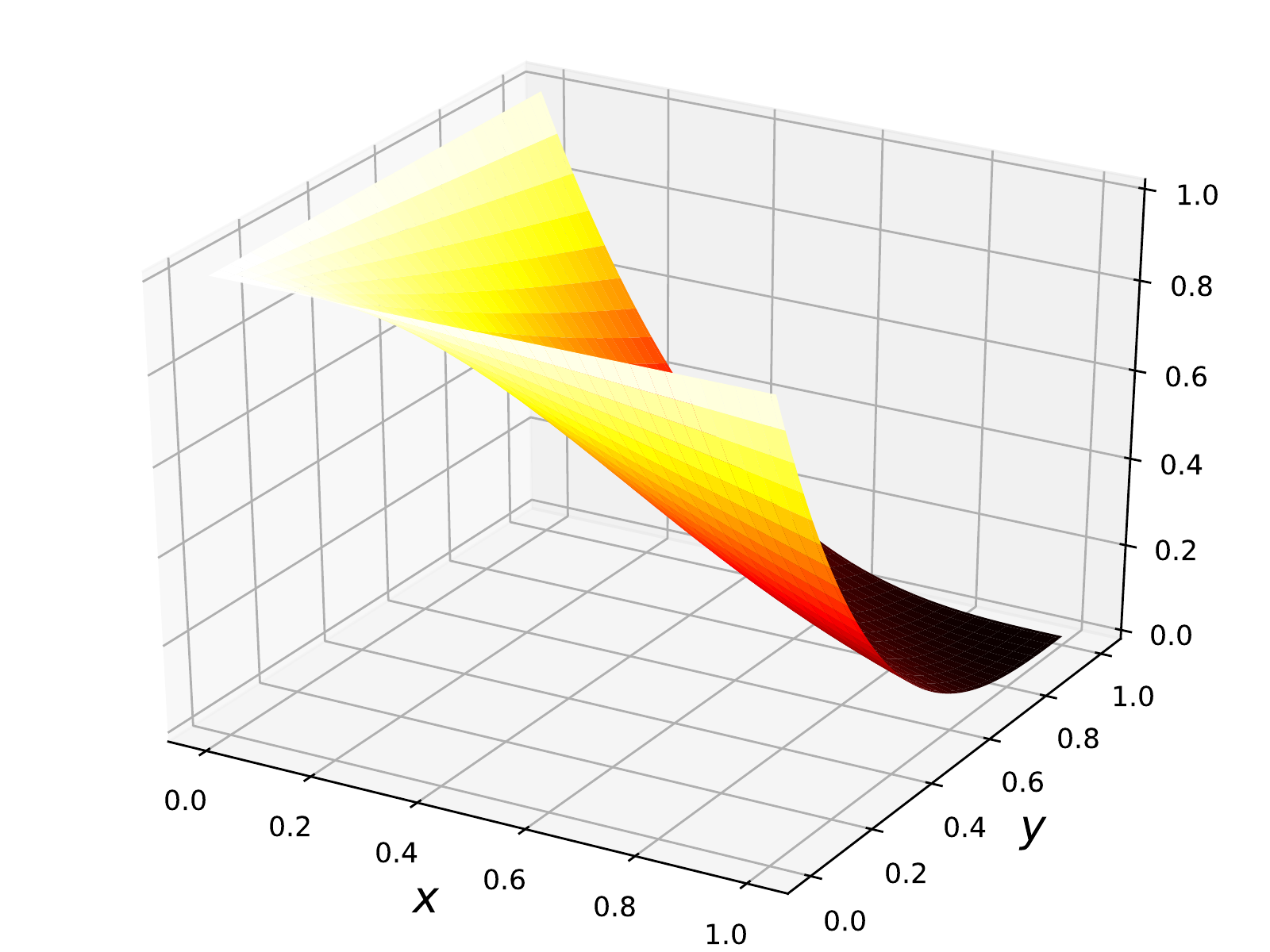}}
    \subfigure[Convergence Behaviors ]{\includegraphics[width=0.28\linewidth]{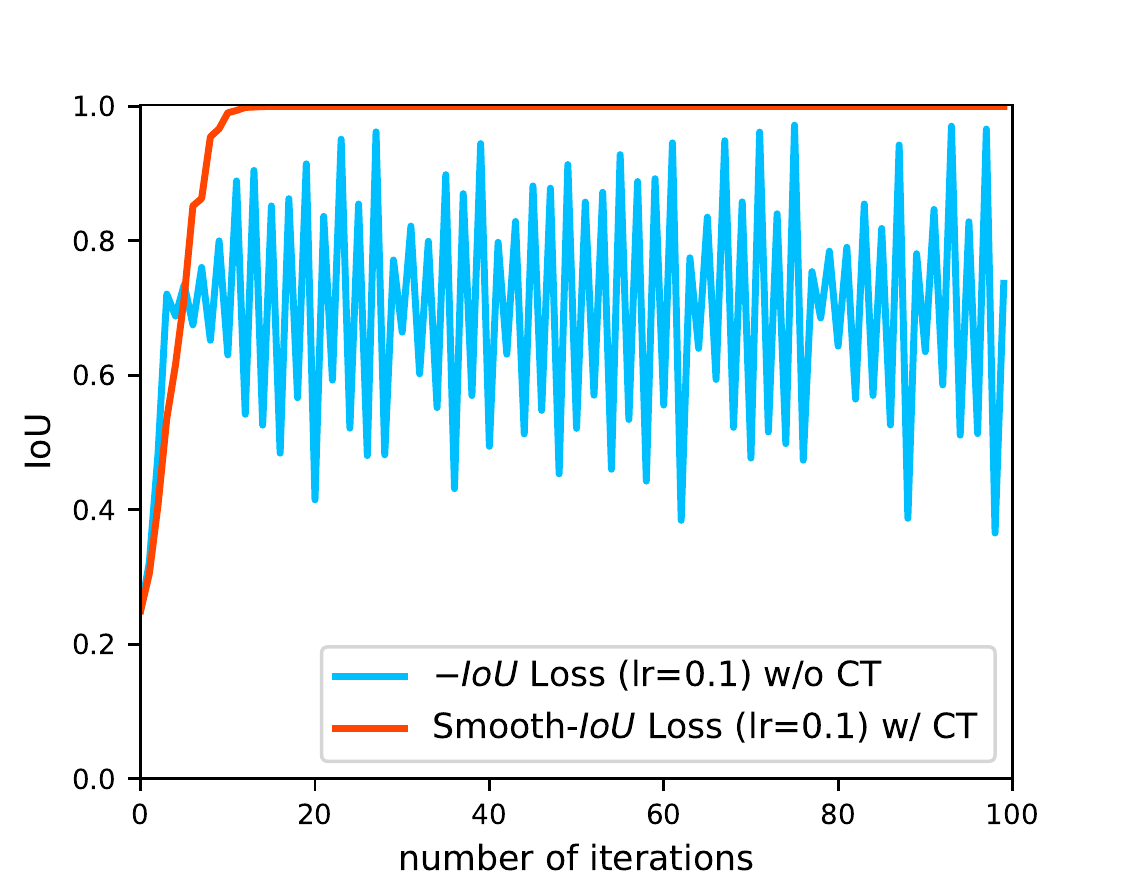}}
    \caption{ \small{Visualization of \small{$- EIoU$} and Smooth-$EIoU$ Loss constructed from \small{$- EIoU$} with CT in Eq. (\ref{Eq.22}) when the targeted box is fixed with $(0, 0, 1, 1)$ and the predicted box varies with $(0, 0, x, y)$. (a) shows the value space of \small{$-EIoU$}.  \small{$- EIoU$} is smaller than zero and not smooth in the neighbourhood of the minimum, and its gradient at the minimum is \emph{non-zero} ;  (b) shows the value space of Smooth-$EIoU$ Loss. After employing CT, Smooth-$EIoU$ Loss is lager than zeros and smooth everywhere, and its gradients  are gradually close to zero when it approaches the minimum. (c) Convergence behaviours of \small{$-EIoU$} and Smooth-$EIoU$ Loss when the targeted box is fixed with $(0, 0, 1, 1)$ and the predicted box is initialized with $(0, 0, 0.5, 0.5)$. Due to the non-zero gradients of \small{$- EIoU$} at the minimum and non-smoothness in its neighbourhood, \small{$-EIoU$} Loss \emph{oscillates} dramatically and cannot converge. In contrast, {Smooth-$EIoU$ Loss quickly converges} to the optimum.} }
    \label{Fig.3}
    \vspace{-1em}
\end{figure*}

Loosely speaking, any a decreasing function w.r.t. $IoU$ can be treated as a localization loss, such as $\frac{1}{IoU}$, {$-IoU$} and $-\ln(IoU)$, {but there are two problems in these simple $IoU$ based losses. \emph{First}, they are not ensured to be always non-negative. \emph{Second}, the gradients of them at the minimum are not zero. } It is well known that (stochastic) gradient methods ideally achieve a minimal point of which the gradient must be zero. Thus,  theoretically, it cannot achieve the minimum if we use these losses in train. To make matters worse, non-zero gradients at the minimum are more likely to make the training process oscillating/non-convergent and even collapsed in practice.
To tackle these problems, we present the convexification technique (CT) to modify the loss and make it practical during the training. It needs two steps:
\begin{enumerate}[(i)]
\item {Add the opposite number of the minimum of the original loss.}
\item {Square the sum above .}
\end{enumerate}

Adopting CT,  any a decreasing functions w.r.t. $IoU$ will become  a well-defined loss, so that it is always non-negative and the gradient at the minimum is zero. Note that  {CT is general, which can be employed to modify any loss not limited the localization loss and make it possess appealing characteristics.}  In this paper, we present a new loss based on the simplest decreasing function {$-EIoU$} w.r.t. $EIoU$.  The minimal value of  $-EIoU$ is $-1$,  so the loss is obtained through CT as follows\footnote{For a more general method, the power order is not limited to 2 but can be any number more than 1, such as ${\cal L}_{\rm Smooth\textrm{-}EIoU} = \left(1- \frac{I_e}{U_e}\right)^{1.5}$.}, \emph{i.e.},
\begin{align}
  {\cal L}_{\rm Smooth\textrm{-}EIoU} = \left(1- \frac{I_e}{U_e}\right)^2
\label{Eq.22}
\end{align}

 CT can smooth  loss functions, so the new loss is referred to as Smooth-$EIoU$ Loss. The new loss  is  also like Focal Loss. CT  leads Smooth-$EIoU$ Loss to possess focal capability that down-weights the gradient of well-localized predicted boxes,  \emph{i.e.},
\begin{equation}
\vspace{-0.6em}
\frac{\partial{\cal L}_{\rm Smooth\textrm{-}EIoU}}{\partial z} = -\left(1- \frac{I_e}{U_e}\right)\frac{\partial\left(\frac{I_e}{U_e}\right)}{\partial z}
\label{Eq:20}
\end{equation}
\noindent where $z$ is any one of $\{x_1^p, y_1^p, x_2^p, y_2^p\}$. It is known $EIoU$ between a well-localized box and the ground-truth box is close to $1$, and then $\left(1- \frac{I_e}{U_e}\right)$ will be close to $0$. Thus, $\frac{\partial{\cal L}_{\rm Smooth\textrm{-}EIoU}}{\partial z}$ will also become very small, which means Smooth-{${EIoU}$} Loss will down-weight easy pair boxes and pay more attention to hard pair boxes in train.

The following example illustrates the importance of CT. Given the targeted bounding box with a tuple $(0, 0, 1, 1)$ and the predicted bounding box with a tuple $(0, 0, x, y)$, the value space of {$-EIoU$} \  and  Smooth-$EIoU$ Loss constructed from {$-EIoU$} with CT are shown in  \ref{Fig.3} (a)-(b).  Smooth-$EIoU$ Loss becomes smooth after employing CT, and then the gradients of the loss are gradually close to zero when approaching minimum, so CT makes Smooth-$EIoU$ Loss  achieves the minimum when applying a gradient descent algorithm, which can be  observed in  \ref{Fig.3} (c).   \ref{Fig.3} (c) shows  the convergence behavior of  $-EIoU$   and Smooth-$EIoU$ Loss when the predicted bounding box starts with the initial value $(0, 0, 0.5, 0.5)$. Not surprisingly, $-EIoU$  oscillates severely and there is no tendency to be converged. In contrast, Smooth-$EIoU$ Loss quickly and smoothly converges to the optimum. Notably,  the steady optimization technique (SOT) that we will elaborate in the next subsection is adopted for $-EIoU$  and Smooth-$EIoU$ Loss in this experiment.

\subsection{Steady Optimization Technique (SOT) }
 For simplicity, we only deduce the partial derivative of Smooth-$EIoU$ Loss in Eq. (\ref{Eq.22}) w.r.t. $x_1^p$ here, and others are similar and presented in the appendix. We first compute the  gradient of $I_e$ w.r.t. $x_1^p$, $i.e.,$
\begin{equation}
\frac{\partial I_e}{\partial x_1^p} = \left\{
 \begin{aligned}
    &y_{\rm min}- y_{\rm max},  ~~~~~~~~{\rm if} ~ x_1^p \ge x_1^t ~{\rm and} ~ x_1 \le x_2, \\
    &2y_0 - y_{\rm max} - y_1,  ~~{\rm if}~  x_1^p \ge x_1^t ~{\rm and}~  x_1 > x_2,  \\
    &0, ~~~~~~~~~~~~~~~~~~~~~~~~~~~\;{\rm if}~ x_1^p < x_1^t.
 \end{aligned}
 \right .
 \label{Eq.23}
\end{equation}

And then we compute the  gradient of $U_e$ w.r.t. $x_1^p$
\begin{equation}
\frac{\partial U_e}{\partial x_1^p} = (y_1^p - y_2^p) - \frac{\partial I_e}{\partial x_1^p}.
\label{Eq.24}
\end{equation}

Finally we obtain the  gradient of  Smooth-$EIoU$ Loss w.r.t. $x_1^p$
\begin{equation}
  \frac{{\partial \cal} L_{\rm{Smooth}\textrm{-}\rm{EIoU}}}{\partial x_1^p} = 2\left(1- \frac{I_e}{U_e}\right)\frac{I_e\frac{\partial U_e}{\partial x_1^p} - \frac{\partial I_e}{\partial x_1^p}U_e  }{U_e^2}.
\label{Eq.25}
\end{equation}

The  partial derivative of Smooth-$EIoU$ Loss w.r.t. $y_1^p$, $x_2^p$ and $y_2^p$ are similar to  Eq. (\ref{Eq.25}) (details please see Appendix.A).  From Eq. (\ref{Eq.17}), Eq. (\ref{Eq.21}), and Eq. (\ref{Eq.24}) ,  we know $I_{\rm e} \propto s$, $U_{\rm e} \propto s^2$, $\frac{\partial I_e}{\partial x_1^p} \propto s$ and  $\frac{\partial U_e}{\partial x_1^p} \propto s$ where $s$ is the size (height or width) of the predicted box $(x_1^p, y_1^p, x_2^p, y_2^p)$. Hence, we analyze   Eq. (\ref{Eq.25}) and find $\frac{{\partial \cal} L_{\rm{Smooth}\textrm{-}\rm{EIoU}}}{\partial x_1^p} \propto \frac{1}{s}$, which means the gradient of  Smooth-$EIoU$ Loss w.r.t. $x_1^p$ is inverse proportional to the size of the predicted box. This inverse proportion will make the loss difficult to converge in train, since when the size of the predict box is large, it means the absolute difference between the targeted box and the predicted box is also large, and then it needs to update with a relatively large step, but if applying gradient in  Eq. (\ref{Eq.25}) to update the variables, the update is small instead. When the size of the boxes is small, it will encounter a similar dilemma. A  good iteratively update for variables should be proportional to the size, just like the gradients of $\ell_2$ Loss. To achieve this goal, we change the update rule for variables of $EIoU$. We take $x_1^p$ for example, \emph{i.e.},
\begin{equation}
\begin{aligned}
\vspace{-5em}
  x_{1_k}^p &=  x_{1_{k-1}}^p - 2\alpha \frac{{\partial \cal} L_{\rm{GI}\textrm{-}\rm{IoU}}}{\partial x_1^p}U_e \vspace{-2em}\\
  &= x_{1_{k-1}}^p -2\alpha\left(1- \frac{I_e}{U_e}\right)\frac{I_e\frac{\partial U_e}{\partial x_1^p} - \frac{\partial I_e}{\partial x_1^p}U_e  }{U_e}
\end{aligned}
\vspace{-0.5em}
\label{Eq.26}
\end{equation}
\noindent where $k$ is the number of iterations and $\alpha$ is the learning rate. Compared with Eq. (\ref{Eq.25}), Eq. (\ref{Eq.26}) multiplies $U_e$ to make sure the new gradient update is proportional to the scale of the boxes.

We call this method  as the steady optimization technique (SOT). This technique seems to be heuristic, but we will theoretically prove it reasonable in the following.

\begin{theorem}
If the gradient of $f(x)$, denoted as $\nabla f(x)$, is Lipschitz continuous, i.e.,
\begin{equation}
\Vert \nabla f(x_1) - \nabla f(x_2)\Vert \le L \Vert x_1 - x_2\Vert_2,
\label{Eq.27}
\end{equation}
 the function $g(x)$ is positive and bounded, i.e., $0 < g(x) \le M$, and the learning rate satisfies $\alpha < \frac{1}{LM}$,  the  update update rule,
\begin{equation}
x_{k+1} = x_k - \alpha g(x_k) \nabla f(x_k),
\label{Eq.28}
\end{equation}
will make $f(x)$ steadily decrease.
\label{The.1}
\end{theorem}

We provide the proof in the appendix. {$U_e$ in our Smooth-$EIoU$ Loss is always  greater than zero. Therefore if we set the learning rate properly, SOT can ensure Smooth-$EIoU$ Loss steadily decreases. From Eq. (\ref{Eq.22}) we know the Smooth-$EIoU$ Loss  is nonnegative and bounded, hence SOT will further guarantee it to be convergent according to the bounded monotonic principle.}

\begin{figure}[ht]
    \vspace{-1em}
    \centering
    \subfigure[]{\includegraphics[width=0.5\linewidth]{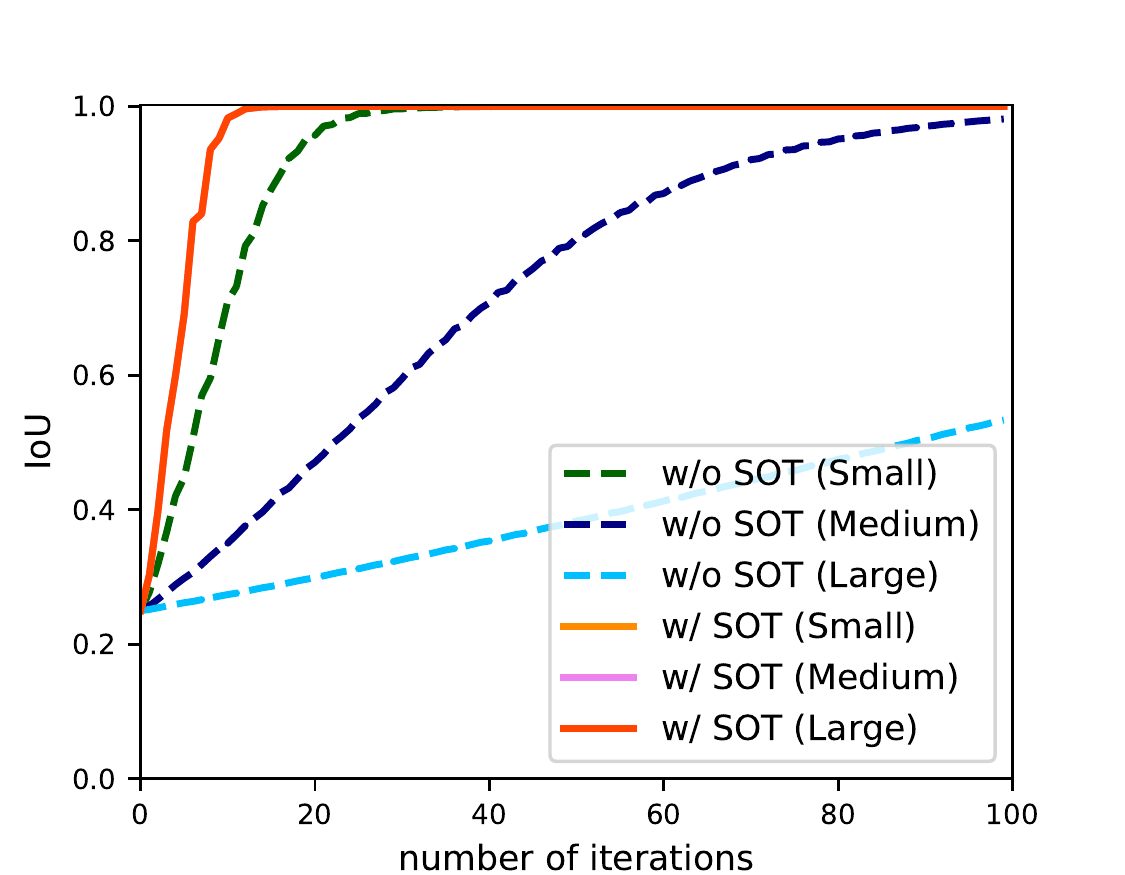}}\hspace{-5 pt}
    \subfigure[]{\includegraphics[width=0.5\linewidth]{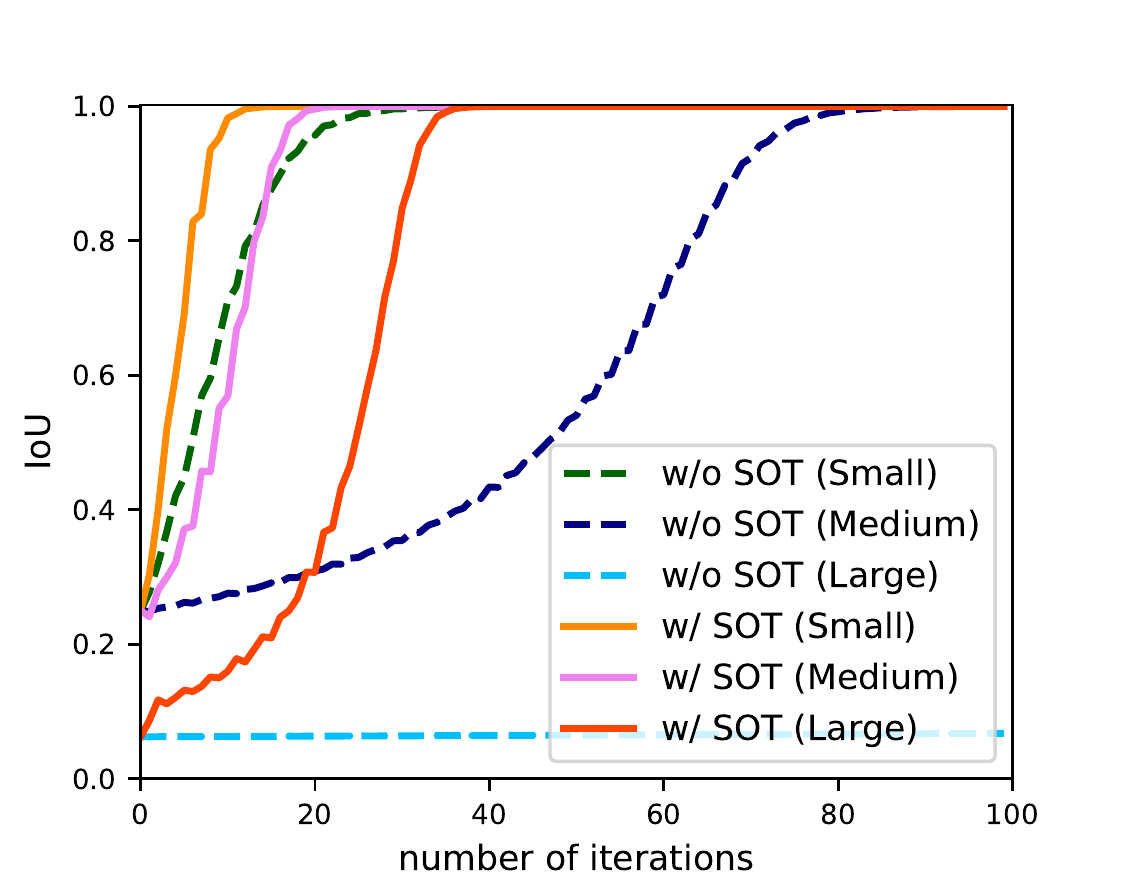}}
    \caption{ \small{ Convergence of Smooth-$EIoU$ Loss optimized with/without SOT: (a) comparisons when the size of the targeted box and predicted box proportionally varies: the targeted box are fixed with $(0, 0, 1, 1)$ (small), $(0, 0, 2, 2)$ (medium) and $(0, 0, 4, 4)$ (large), and the initial value of the predicted is proportionally set as $(0, 0, 0.5, 0.5)$ (small), $(0, 0, 1, 1)$ (medium) and $(0, 0, 2, 2)$ (large). The converges tendency of Smooth-$EIoU$ Loss with SOT is completely the same regardless of the size, while Smooth-$EIoU$ Loss {without SOT is very sensitive to the varied size.} The larger the size is,  the \emph{slower} the convergence rate is, just like what we analyzed. (b) comparison when only the size of the predicted box varies: the targeted box are fixed with $(0, 0, 1, 1)$, and the initial value of the predicted is set as $(0, 0, 0.5, 0.5)$ (small), $(0, 0, 2, 2)$ (medium) and $(0, 0, 4, 4)$ (large). Smooth-$EIoU$ Loss with {SOT still can quickly converge}, but Smooth-$EIoU$ Loss without SOT is more sensitive to the size under this circumstance. When the initial value of the predicted box set as $(0, 0, 4, 4)$, it is even \emph{trapped} and cannot move to the target box.} }
    \label{Fig.4}
    \vspace{-0.5em}
\end{figure}

According to Theorem \ref{The.1},   {SOT is very general and can be applied to optimize many types of losses for steady convergence, including but not limited to fractional losses of which its gradients is not linearly proportional to the size.}

 We design two examples in \ref{Fig.4} to further demonstrate the superiority of SOT. As shown in  \ref{Fig.4}, SOT will make the convergence of the loss steady, and it is robust to the size of the initial predicted box and the targeted box. When the size of the targeted box and predicted box proportionally varies, the targeted box are fixed with $(0, 0, 1, 1)$ (small), $(0, 0, 2, 2)$ (medium) and $(0, 0, 4, 4)$ (large), and the initial value of the predicted is proportionally set as $(0, 0, 0.5, 0.5)$ (small), $(0, 0, 1, 1)$ (medium) and $(0, 0, 2, 2)$ (large). The converges tendency of Smooth-$EIoU$ Loss with SOT is completely the same regardless of the size, while Smooth-$EIoU$ Loss {without SOT is very sensitive to the varied size.} The larger the size is,  the \emph{slower} the convergence rate is, just like what we analyzed above.  When only the size of the predicted box varies, the targeted box are fixed with $(0, 0, 1, 1)$, and the initial value of the predicted is set as $(0, 0, 0.5, 0.5)$ (small), $(0, 0, 2, 2)$ (medium) and $(0, 0, 4, 4)$ (large). Smooth-$EIoU$ Loss with {SOT still can quickly converge}, but Smooth-$EIoU$ Loss without SOT is more sensitive to the size under this circumstance. When the initial value of the predicted box set as $(0, 0, 4, 4)$, it is even \emph{trapped} and cannot move to the target box.

\subsection{IoU Head}
In \cite{IoUNet_2018} it has demonstrated that there is a misalignment between classification confidence and localization accuracy, and utilizing precisely predicted $IoU$ scores of bounding boxes to guide NMS will largely alleviate this problem. Taking advantage of the existing ground-truth $IoU$ calculated in Smooth-$EIoU$ Loss,  we add IoU Head and train it to predict accurate $IoU$ scores. It is known  $IoU$ distributes over $[0, 1]$, so we first utilize the sigmoid function to compress the predicted $IoU$ score to $[0, 1]$, and then a Kullback-Leibler (KL) divergence loss is employed in train, \emph{i.e.},
\begin{align}
&q_p(x) = {\rm Sigmoid}(x), \\
&{\cal L}_{KL} = q_g\log\frac{q_g}{q_p(x)} + (1 - q_g)\log\frac{1-q_g}{1-q_p(x)},
\end{align}
where $x$ is the output of IoU Head, $q_p(x)$ is the predicted $IoU$ score and $q_g$ is the ground-truth $IoU$ score that is generated in Smooth-$EIoU$ Loss.

Note that IoU Head is a single layer, and it shares most parameters with the classification  head and the bounding-box  head. Hence, it will increase little computational cost in train and test.

\noindent\textbf{Differences From IoU-Net.} \cite{IoUNet_2018} pioneeringly proposed  IoU-Net learning to predict $IoU$ to promote the localization accuracy. However,  there are still some significant differences between our IoU Head and the IoU-Net. \emph{Firstly}, we used a KL loss that is widely proven to be effective for deep neural networks rather than a squared loss. \emph{Secondly}, it needs to manually construct synthetical bounding-box sets to train  IoU-Net individually besides training the main branches of classification and localization, while our IoU Head can  seamlessly embed  to the existing network and be trained end-to-end. \emph{Thirdly},   IoU Head is much lighter than  IoU-Net.  IoU-Net is an individual subnet and works parallelly  with the classification subnet and the localization subnet, while  IoU Head is a single-layer branch and shares most layers with the main branches. Architectures of IoU-Net and IoU Head are visually illustrated in Fig \ref{Figure.6}. \emph{Fourthly}, the ground-truth $IoU$ used in IoU Head is generated by the localization head, so  IoU Head and  Localization Head are closely interrelated with better cooperativity, and the effect of "$1+1>2$" between them is shown in Table 1. But IoU-Net has little relation with the localization head.

\begin{figure}[htbp]
    \small
    \centering
    \subfigure[IoU-Net]{\includegraphics[width=0.8\linewidth]{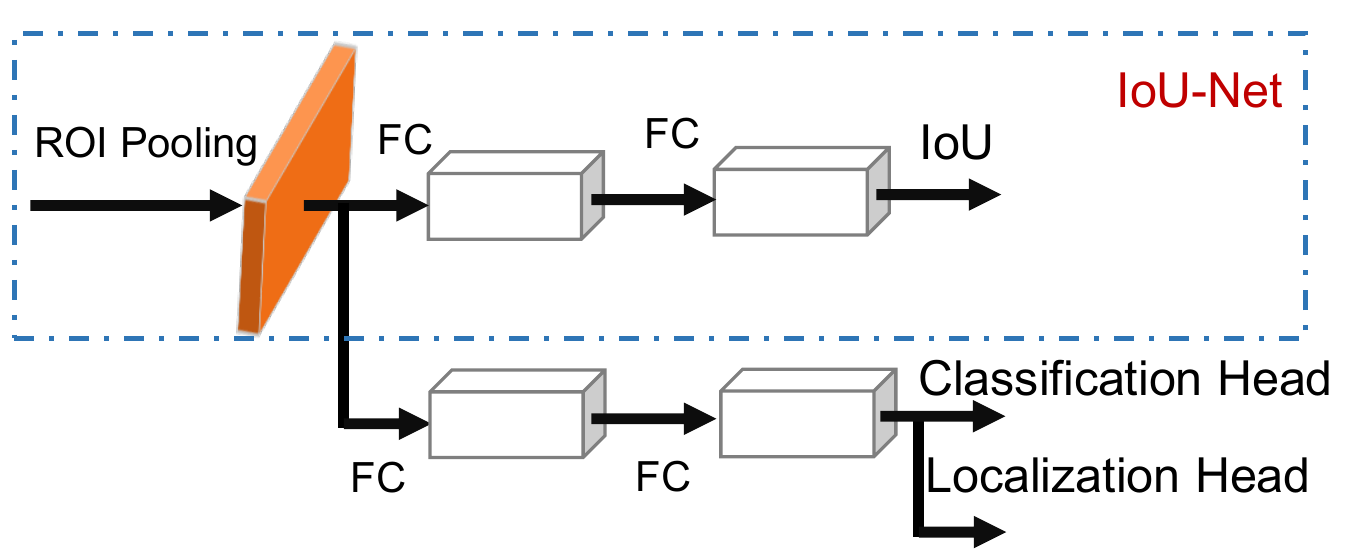}}
    \subfigure[IoU Head]{\includegraphics[width=0.8\linewidth]{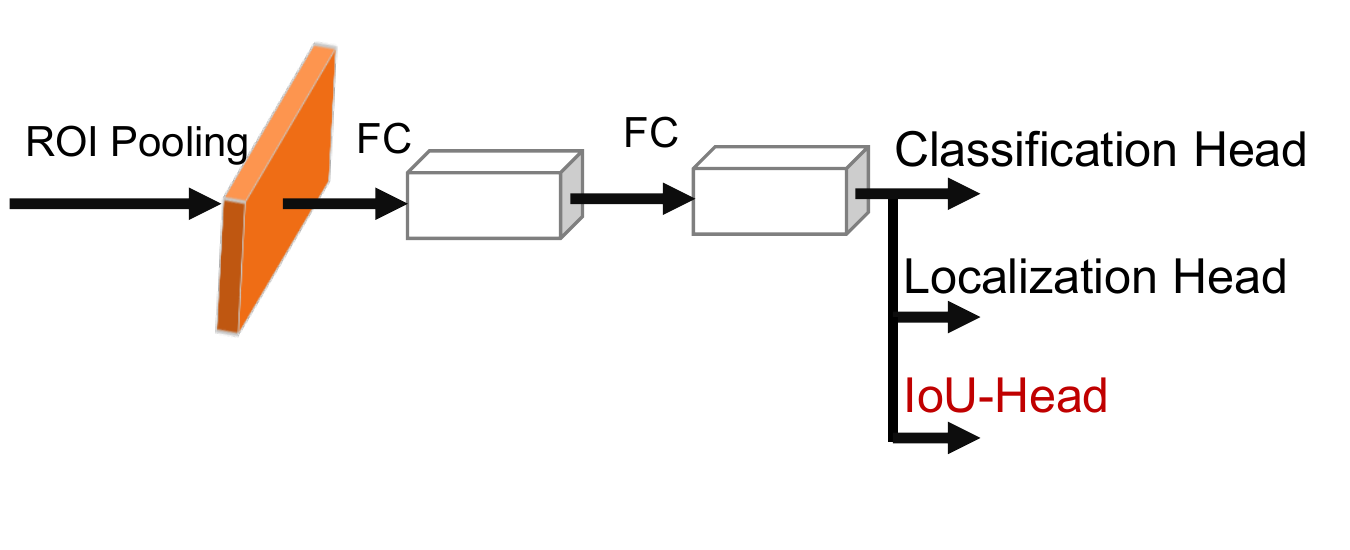}}
    \vspace{-5pt}
    {\caption{\small{ The network architectures of IoU-Net and IoU Head.}}
    \label{Figure.6}}
    \vspace{-10pt}
\end{figure}

\section{Implementation}

In modern deep CNN based detectors, the neural network does not directly estimate the coordinates of the bounding box, and instead it predicts the normalized difference value between the corresponding coordinates of the anchor or proposal box (henceforth, we only use anchor box for simplicity ) and the targeted box, and the normalization value is the width and height of anchor box. We adopt a similar strategy to generate the predicted box, but we uniformly employ the square root of the area of the anchor box to normalize all the coordinates rather than independently normalize them with the corresponding coordinate of the anchor, since the former will keep the width-height ratio of the predicted box and targeted box. Implementation details please see Algorithm 1.
\renewcommand\arraystretch{1.3}
\begin{table}[!htbp]
\small
 \centering
  \begin{threeparttable}
    \begin{tabular}{p{8cm}l}
      \hline
      ~~~~~~~~~~~~~~~~~~~~\textbf{Algorithm 1.} Training $EIoU$ Loss \\
      \hline
       \textbf{Input}: the anchor box $(x_1^a, y_1^a, x_2^a, y_2^a)$, the target box $(x_1^t, y_1^t, x_2^t, y_2^t)$ and the CNN predicted
        normalized difference value  $(x_{1_0}^d, y_{1_0}^d, x_{2_0}^d, y_{2_0}^d)$  \\
       \textbf{Output}: the $EIoU$ Loss ${\cal L}_{\rm{GI}\textrm{-}\rm{IoU}}$ \\
        Compute  $S = \sqrt{(x_2^a - x_1^a)(x_2^a - x_1^a)}$ \\
        Compute   $(x_1^{a_n}, y_1^{a_n}, x_2^{a_n}, y_2^{a_n}) = (\frac{x_1^a}{S}, \frac{y_1^a}{S}, \frac{x_2^a}{S}, \frac{y_2^a}{S}) $ \\
        Compute   $(x_1^{t_n}, y_1^{t_n}, x_2^{t_n}, y_2^{t_n}) = (\frac{x_1^t}{S}, \frac{y_1^t}{S}, \frac{x_2^t}{S}, \frac{y_2^t}{S}) $ \\
       \textbf{while} \emph{not convergence} \textbf{do} \\
       \quad Compute  $(x_{1_k}^p, y_{1_k}^p, x_{2_k}^p, y_{2_k}^p) = (x_{1_k}^d, y_{1_k}^d, x_{2_k}^d, y_{2_k}^d)$ \\
       \quad \quad $ + (x_1^{a_n}, y_1^{a_n}, x_2^{a_n}, y_2^{a_n})$ \\
       \quad Using Eq. (\ref{Eq.1})- (\ref{Eq.4}), (\ref{Eq.11})- (\ref{Eq.22}) to compute ${\cal L}_{\rm{GI}\textrm{-}\rm{IoU}}$ \\
       \quad Using Eq. (\ref{Eq.23})- (\ref{Eq.25}) to compute $(\frac{\partial \cal L}{\partial x_{1_k}^p}, \frac{\partial \cal L}{\partial y_{1_k}^p}, \frac{\partial \cal L}{\partial x_{2_k}^p}, \frac{\partial \cal L}{\partial y_{2_k}^p} )$ \\
       \quad $(\frac{\partial \cal L}{\partial x_{1_k}^d}, \frac{\partial \cal L}{\partial y_{1_k}^d}, \frac{\partial \cal L}{\partial x_{2_k}^d}, \frac{\partial \cal L}{\partial y_{2_k}^d} ) = (\frac{\partial \cal L}{\partial x_{1_k}^p}, \frac{\partial \cal L}{\partial y_{1_k}^p}, \frac{\partial \cal L}{\partial x_{2_k}^p}, \frac{\partial \cal L}{\partial y_{2_k}^p} )$ \\
       \quad Using Eq. (\ref{Eq.26}) to update $(x_{1_{k+1}}^d, y_{1_{k+1}}^d, x_{2_{k+1}}^d, y_{2_{k+1}}^d)$

       \textbf{end while}\\
      \hline
    \end{tabular}
  \end{threeparttable}
\end{table}

\renewcommand\arraystretch{1.2}
\begin{table*}[!hptb]
\small
\centering
\caption{ Ablation study by using Faster R-CNN with ResNet50 + FPN as the backbone. Models are trained on the union set of VOC\_2007\_trainval and VOC\_2012\_trainval. The results are reported on the set of VOC\_2007\_test.}
\begin{tabular}{|p{0.2cm}<{\centering} p{1.5cm}<{\centering} p{0.8cm}<{\centering} p{0.8cm}<{\centering} p{0.8cm}<{\centering} p{0.8cm}<{\centering} p{0.8cm}<{\centering}p{1.5cm}<{\centering}|l|}
\hline
 &Smooth-$l_1$ &$SIoU$ &$GIoU$ &$EIoU$ &CT &SOT  &IoU Head &AP  \\
\hline
\hline
\textcircled{1} &\checkmark &- &- &- &- &- &- &45.5 (Smooth-$l_1$ Loss, Baseline) \\
\textcircled{2} &- &\checkmark &-  &- &- &- &- &46.6 (Standard $IoU$ Loss ) \\
\textcircled{3} &- &- &\checkmark   &- &- &- &- &46.9 ($GIoU$ Loss) \\
\hline
\hline
\textcircled{4} &\checkmark  &- &-  &- &- &- &\checkmark &46.2 ( Smooth-$l_1$ Loss with IoU Head) \\
\textcircled{5}&- &- &- &\checkmark &- &- &-  &47.5  ($EIoU$ Loss) \\
\textcircled{6}&- &- &- &\checkmark  &\checkmark &- &- &47.9 ($EIoU$ Loss with CT )\\
\textcircled{7}&- &- &- &\checkmark &\checkmark  &\checkmark  &- &48.2 ($EIoU$ Loss with CT, SOT)\\
\textcircled{8}&- &- &- &\checkmark &\checkmark &\checkmark &\checkmark &{49.7} ($EIoU$ Loss with CT, SOT and IoU Head)\\
\hline
\end{tabular}
\label{Tab.3}
\end{table*}

\section{Experiment}

\subsection{Experimental Setting}

All the experiments are conducted  on the benchmark datasets -- PASCAL VOC and MS COCO. Detectors are implemented in Facebook AI Research's Detectron system \cite{Detectron2018}. Following the default settings in Detectron, we trained all the detectors on 8 NVIDIA P100 GPUs. Each mini-batch totally contains 16 images which are uniformly distributed to 8 GPUs. Input images are resized to 500 and 800 pixels along the short side on PASCAL VOC and MS COCO, respectively. No other data augmentation except of the standard horizontal image flipping is employed. Standard SGD with weight decay of $0.0001$ and  momentum of $0.9$ is adopted. We train the detectors with $20k$ iterations for PASCAL VOC and $90k$($180k$) iterations for MS COCO, and the learning rate is set to 0.02 at the begin and then decreased by a factor of $0.1$ after $12k$ and $18k$ for PASCAL VOC and $60k$($120k$) and $80k$($160k$) iterations for MS COCO, respectively.  We comply with the MS COCO evaluation protocol to report the experimental results.

\subsection{Ablation Study}

We implement ablation experiments on PASCAL VOC to clarify the contributions of the proposed $EIoU$, CT, SOT and IoU Head, and the results are reported in Table \ref{Tab.3}.

As shown in Table \ref{Tab.3}, with the standard {$SIoU$} based loss replacing the baseline Smooth-$\ell_1$ Loss, the performance is improved to some extents (+1.1\% mAP, comparing \textcircled{1} and \textcircled{2}). Substituting {$SIoU$} with {$GIoU$} further boost the performance with scores  +0.3\% mAP (comparing \textcircled{2} and \textcircled{3}), which is consistence with the results in Table 5 in \cite{GIOU_2019}. Comparing with {$GIoU$}, individually equipping $EIoU$ can bring more substantial improvement(+0.9 \% mAP,  \textcircled{2} and \textcircled{5}), which indicates $EIoU$ may be more piratically powerful than $GIoU$.   With the help of CT, the performance is continually promoted (+0.4 \% mAP,  \textcircled{5} and \textcircled{6}). Exploiting SOT in train further receives a gain of +0.3 \% mAP scores (comparing \textcircled{6} and \textcircled{7}). Adding IoU Head to the net significantly improves the performance (+1.5\% mAP, comparing \textcircled{7} and \textcircled{8}). Interestingly, $EIoU$ Loss with IoU Head can generate better cooperativity than  Smooth-$l_1$ Loss with IoU Head(+1.5\% mAP vs +0.7\% mAP, comparing \textcircled{7} , \textcircled{8} and comparing \textcircled{1} , \textcircled{4}). The reason for it is that IoU Head has close relation to a IoU related  loss, so they can receive the effect of "$1+1>2$". Totally, the proposed systematical method including $EIoU$, CT, SOT and IoU Head yields significant gains, which is 4.2\% higher than the baseline Smooth-$l_1$ Loss that is overwhelmingly used in popular detectors (comparing \textcircled{1} and \textcircled{8}).

\renewcommand\arraystretch{1.1}
\begin{table*}[!hptb]
\small
\centering
\caption{ Comparisons of Average Precision(AP) of Smooth-$\ell_1$ Loss, $GIoU$ Loss and $EIoU$ Loss attaching to RetinaNet and Faster-RCNN with ResNet50+FPN as the backbone. Models are trained on the union set of VOC\_2007\_trainval and VOC\_2012\_trainval. The results are reported on the set of VOC\_2007\_test.}
\begin{tabular}{|l|p{3.0cm}<{\centering}|p{1.0cm}<{\centering}|p{1.0cm}<{\centering}p{1.0cm}<{\centering}p{1.0cm}<{\centering}|
p{1.0cm}<{\centering}p{1.0cm}<{\centering}p{1.0cm}<{\centering}|}
\hline
loss  &Net   &$\rm{mAP}$ &$\rm{AP_{50}}$   &$\rm{AP_{75}}$  &$\rm{AP_{90}}$ &$\rm{AP_{S}}$   &$\rm{AP_{M}}$  &$\rm{AP_{L}}$\\
\hline
\hline
Smooth-$\ell_1$ Loss (Baseline) \cite{FPN2017} &RetinaNet  &44.2 &{70.5}  &47.5  &25.0 &9.8 &28.3 &54.3 \\
$GIoU$ Loss \cite{GIOU_2019} &RetinaNet  &45.2 &{70.0}  &48.2  &28.8 &9.9 &29.1 &55.5 \\
$CIoU$ Loss \cite{DIoU2020} &RetinaNet  &45.9 &{70.3}  &50.1  &30.3 &10.1 &30.2 &55.9 \\
Ours  &RetinaNet      &\textbf{46.4} &\textbf{71.1}  &\textbf{49.1}  &\textbf{32.2} &\textbf{10.5}  &\textbf{31.4}  &\textbf{56.3} \\

\hline
\hline
Smooth-$\ell_1$ Loss (Baseline) \cite{FPN2017}  &Faster-RCNN+FPN  &45.5 &72.6  &49.8  &25.2 &10.0  &29.5  &55.6\\
$GIoU$ Loss \cite{GIOU_2019} &Faster-RCNN+FPN  &46.9 &73.1  &50.8  &28.6 &9.6  &31.0  &57.2\\
$CIoU$ Loss \cite{DIoU2020} &Faster-RCNN+FPN    &48.0 &73.5  &51.2  &28.8 &9.8  &31.8  &57.9 \\
Ours  &Faster-RCNN+FPN        &\textbf{49.7} &\textbf{73.7}  &\textbf{54.1}  &\textbf{33.4} &\textbf{11.9}  &\textbf{32.8}  &\textbf{59.8}\\
\hline
\end{tabular}
\label{Tab.1}
\end{table*}
\renewcommand\arraystretch{1.1}
\begin{table*}[!hptb]
\small
\centering
\caption{ Comparisons of Average Precision(AP) of of Smooth-$\ell_1$ Loss, $GIoU$ Loss and $EIoU$ Loss attaching to RetinaNet and Faster-RCNN with Res50 + FPN as the backbone.  Models are trained on the union set of COCO\_2017\_train. The results are reported on the set of COCO\_2017\_val.}
\begin{tabular}{|l|p{3.0cm}<{\centering}|p{1.0cm}<{\centering}|p{1.0cm}<{\centering}p{1.0cm}<{\centering}p{1.0cm}<{\centering}|
p{1.0cm}<{\centering}p{1.0cm}<{\centering}p{1.0cm}<{\centering}|}
\hline
loss  &Net   &$\rm{mAP}$ &$\rm{AP_{50}}$   &$\rm{AP_{75}}$  &$\rm{AP_{90}}$ &$\rm{AP_{S}}$   &$\rm{AP_{M}}$  &$\rm{AP_{L}}$\\
\hline
\hline
Smooth-$\ell_1$ Loss (Baseline) \cite{FPN2017} &RetinaNet  &35.7 &54.7  &38.5  &22.8 &19.5  &39.9 &47.5 \\
$GIoU$ Loss \cite{GIOU_2019} &RetinaNet  &36.2 &54.5  &38.9  &24.4 &19.6  &40.3  &48.3\\
$CIoU$ Loss \cite{DIoU2020} &RetinaNet  &36.7 &54.7  &39.3  &25.0 &20.1  &40.9  &48.9\\
Ours &RetinaNet         &\textbf{37.1} &\textbf{54.8}  &\textbf{39.6}  &\textbf{25.7} &\textbf{20.6}  &\textbf{41.3}  &\textbf{49.2} \\
\hline
\hline
Smooth-$\ell_1$ Loss (Baseline) \cite{FPN2017} &Faster-RCNN+FPN  &36.7 &{58.5}  &39.6  &21.2  &{21.1}  &39.8  &48.1\\
$GIoU$ Loss  \cite{GIOU_2019}  &Faster-RCNN+FPN  &37.8 &58.4  &41.0  &24.8 &21.2  &40.9  &49.8 \\
$CIoU$ Loss \cite{DIoU2020}   &Faster-RCNN+FPN  &38.4 &58.3  &41.5  &25.2 &21.2  &41.3  &50.2 \\
Ours  &Faster-RCNN+FPN           &\textbf{39.0} &\textbf{58.7}  &\textbf{42.1}  &\textbf{26.5} &\textbf{22.3}  &\textbf{41.7}  &\textbf{50.7} \\
\hline
\end{tabular}
\label{Tab.2}
\end{table*}

\begin{figure}[t]
    \vspace{-1em}
    \centering
    \subfigure[]{\includegraphics[width=0.5\linewidth]{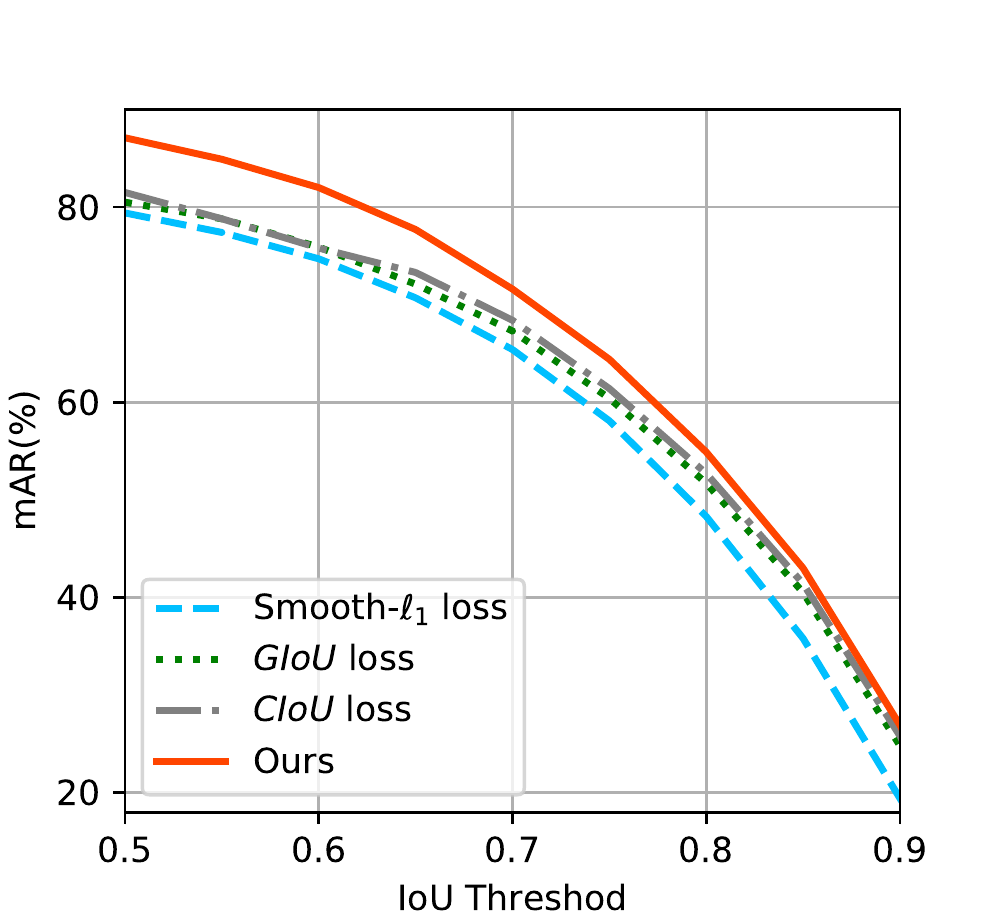}}  \hspace{-10 pt}
    \subfigure[]{\includegraphics[width=0.5\linewidth]{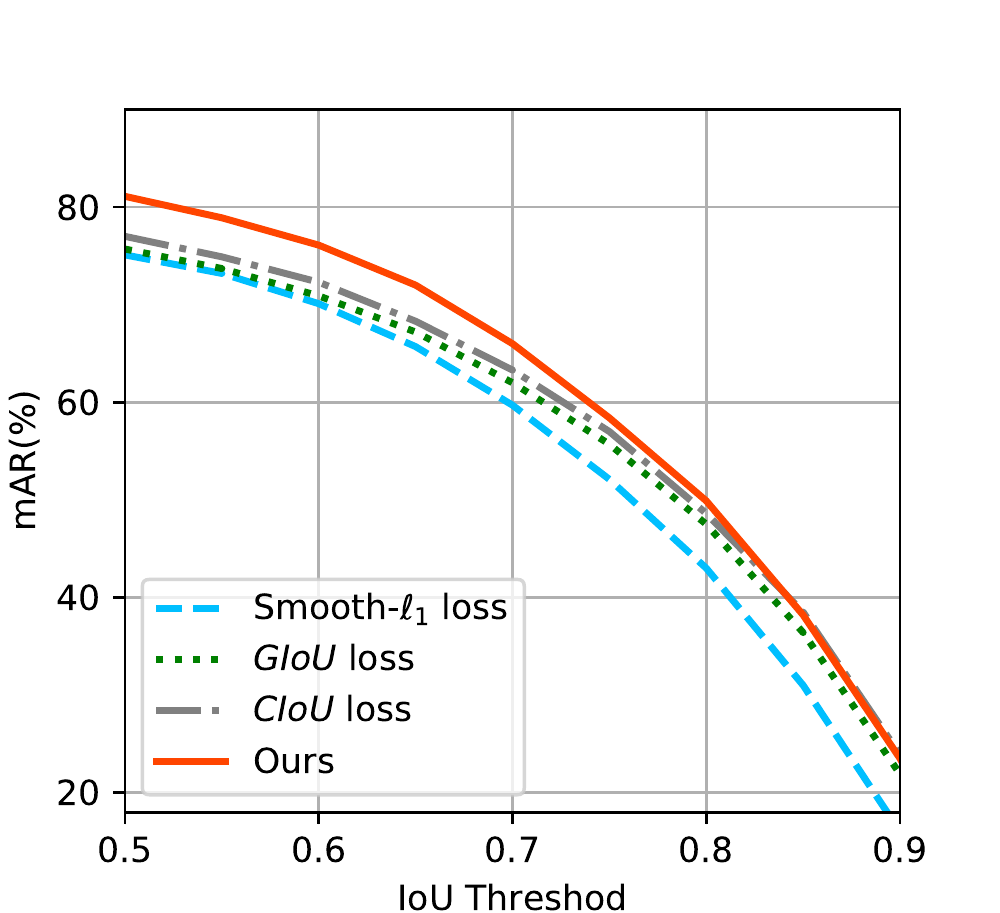}}
    \caption{\small{IoU threshold against Average Recall(AR) of Faster RCNN + FPN with Smooth-$\ell_1$ Loss, $GIoU$ Loss and our approach on (a) PASCAL VOC and (b)MS COCO. Models are trained
on (a) the union set of VOC 2007 trainval and VOC 2012 trainval and (b) the set of COCO 2017 train. The results are demonstrated on (a) the set of VOC 2007 test and (b) the COCO 2017 val.}}
   \vspace{-1em}
    \label{Fig.5}
\end{figure}

\renewcommand\arraystretch{1.1}
\begin{table*}[!hptb]
\small
\centering
\caption{  Performance of state-of-the-art detectors on the set of COCO test-dev.  Our Model is trained on the  set of COCO\_2017\_train.}
\begin{tabular}{|l|l| p{1.2cm}<{\centering}|p{1.2cm}<{\centering}p{1.2cm}<{\centering}p{1.2cm}<{\centering}|p{1.2cm}<{\centering} p{1.2cm}<{\centering}p{1.2cm}<{\centering}|}
\hline
Method  &Backbone &Year &$\rm{mAP}$ &$\rm{AP_{50}}$  &$\rm{AP_{75}}$ & $\rm{AP_S}$ & $\rm{AP_M}$ & $\rm{AP_L}$\\
\hline
\hline
YOLOv3 \cite{YOLOv3_2018}  &DarkNet-53 &2018  &33.0 &57.9 &34.4 &18.3 &35.4  &41.9\\
SSD513 \cite{SSD_2016}  &ResNet-101  &2016 &31.2 &50.4 &33.3 &10.2 &34.5  &49.8\\
RetinaNet800\cite{Focal_loss_retinanet_2017} &ResNeXt-101 &2017  &39.1 &59.1 &42.3 &21.8 &42.7  &50.2\\
DSSD513 \cite{DSSD_2017} &ResNet-101 & 2017 &33.2 &53.3 &35.2 &13.0 &35.4  &51.1\\
RefineDet512\cite{RefineDet2018} &ResNet-101 & 2018 &36.4 &57.5 &39.5 &16.6 &39.9  &51.4\\
CornerNet511\cite{CornerNet_2018} &Hourglass-104 &2018 &40.5 &56.5 &43.1 &19.4 &42.7 &53.9 \\
CenterNet \cite{CornerNet_2018} &Hourglass-104 &2019 &42.1 &61.1 &45.9 &24.1 &45.5 &52.8 \\
FCOS\cite{FCOS2019} &ResNeXt-101 &2019 &43.2 &62.8 & 46.6 &26.5 &46.2 &53.3 \\
\hline
\hline
Faster-R-CNN +++ \cite{ResNet2016} &ResNet-101 &2016 &34.9 &55.7 &37.4 &15.6 &38.7  &50.9\\
Faster-RCNN w FPN \cite{FPN2017} &ResNet-101  &2016 &36.2 &59.1 &39.0 &18.2 &39.0  &48.2\\
Mask R-CNN \cite{Mask_RCNN_2017}  &ResNeXt-101 &2017 &39.8 &{62.3} &43.4 &22.1 &43.2  &51.2\\
DetNet \cite{DetNet_2018}     &DetNet  &2018     &40.3 &62.1 &43.8 &{23.6} &42.6 &50.0 \\
IoU-Net \cite{IoUNet_2018} &ResNet-101 &2018 &40.6 &59.0 &- &- &-  &-\\
TridenNet w \emph{2fc} \cite{TrideNet2019} &ResNet-101 &2019  &42.0 &63.5 &45.5 &24.9 &47.0 &56.9 \\
Grid R-CNN \cite{GridRCNN2019}  &ResNeXt-101 &2019 &43.2 &63.0 &46.6 &25.1 &46.5 &55.2 \\
\hline
\hline
Ours  &ResNet-101 &2019 &{42.2} &61.8 &{46.1} &24.4 &{45.2}  &{55.4} \\
Ours  &ResNeXt-101 &2019 &\textbf{44.1} &\textbf{63.7} &\textbf{47.6} &\textbf{26.8} &\textbf{47.6}  &\textbf{57.1} \\
\hline

\end{tabular}
\vspace{-0.5em}
\label{Tab.4}
\end{table*}

\subsection{Comparison to the Related Localization Losses }
The proposed systematical method is mainly built on localization loss, so we will extensively compare the proposed method to widely-used Smooth-$\ell_1$ Loss  and the related $GIoU$ Loss \cite{GIOU_2019} and $CIoU$ \cite{DIoU2020} Loss in this subsection. For simplicity, our systematical method is referred to as $EIoU$ Loss henceforth.  All the losses are attached to RetinaNet (that is a typical one-stage model)  and Faster-RCNN (that is a typical  two-stage detection model) during  training. Overall Mean average Precision(mAP) for all the three losses is reported in  Table \ref{Tab.1}-\ref{Tab.2}. Besides,  the results of Average Precision (AP) at IoU thresholds: $[0.5, 0.75, 0.90]$ and for individual small-size, medium-size and large-size objects are also listed for detailed comparison.

As shown in Table \ref{Tab.1} and \ref{Tab.2}, compared with Smooth-$\ell_1$ Loss and $GIoU$ Loss,   $EIoU$ Loss in one-stage and two-stage detectors can steadily yield gains on PASCAL VOC and MS COCO. Specifically, for the baseline Smooth-$\ell_1$ Loss that is dominant in popular detectors,  our approach combining Faster R-CNN   substantially boosts $4.2\%$ AP and $1.2\%$ mAP on PASCAL VOC  and COCO, respectively. When comparing with  $GIoU$ Loss,  $EIoU$ loss can still consistently surpass it by a more than $2.0\%$ margin on PASCAL VOC and an $1.0\%$ margin on COCO.

There is an interesting phenomenon that  when the $IoU$ threshold is set to $0.5$ ,  the performance of our approach is close to Smooth-$\ell_1$ Loss. However, when the threshold grows higher,  $EIoU$ Loss  gradually outperforms  Smooth-$\ell_1$ Loss and $GIOU$ Loss,. Especially {{at $\rm{AP}_{90}$, comparing with Smooth-$\ell_1$ Loss, $EIoU$ Loss improves 8.2\% on PASCAL VOC dataset and  5.3\% on  MS COCO dataset }}. The reason for it is $EIoU$ Loss can help a detector to predict more accurate bounds than  Smooth-$\ell_1$ Loss. It is known there is a gap between Smooth-$\ell_1$ Loss and the final evaluation IoU, and the relative gap is enlarging as two boxes are gradually matched, while $EIoU$ is exactly equivalent to $IoU$ when two boxes are overlapping. Moreover, Smooth-$\ell_1$  Loss  decreases quicker than $EIoU$ Loss as two boxes are gradually matched, so during training Smooth-$\ell_1$  Loss commonly gives less attention to better matched pair-boxes. Therefore, comparing to Smooth-$\ell_1$ Loss, $EIoU$ Loss will receive more gains when the final evaluation metric (IoU) is stricter.

Another phenomenon observed from  Table \ref{Tab.1} and \ref{Tab.2} is that $EIoU$ Loss seems to be superior to detect small-size objects, comparing to  {$GIoU$} Loss. Although the overall performance of {$GIoU$} Loss is 1.4\% higher than   Smooth-$\ell_1$ Loss on  PASCAL VOC dataset with Faster-RCNN, but Smooth-$\ell_1$ Loss  and {$GIoU$} Loss obtain similar scores (10.0\% and 9.6\%) for small-size objects, which means {$GIoU$} Loss is still weak to detect small-size objects. $EIoU$ Loss achieves 11.9 \% under the same conditions. The superiority of $EIoU$ loss to detect small-size objects stems from the IoU predict head. In post-processing, conventionally we use classification confidence to guide non-maximum suppression (NMS) to filter redundant bounding boxes. Commonly, the correlation of classification confidence and localization confidence is weaker when detecting smaller objects. In our method, we use the predicted IoU confidence to correct the bias of classification confidence and localization confidence. Hence, our method has a better capacity for finding smaller objects.

Additionally, in terms of improvement,  Faster-RCNN + FPN with $EIoU$ Loss  performs better than RetinaNet with $EIoU$ Loss. It may be due to that there are denser anchor boxes in RetinaNet. Hence it is not so difficult to exactly regress the targeted boxes for Smooth-$\ell_1$ Loss.

As shown in Fig \ref{Fig.5},  the superior performance of $EIoU$ Loss for \emph{Average Recall (AR)} are more obvious than that for AP  across the different value of $IoU$ threshold, which means $EIoU$ Loss is more powerful to find more objects, comparing with the popular localization losses.

\subsection{Comparisons to State-of-the-Art Detectors}
We evaluate $EIoU$ Loss attached to FPN on the MS COCO 2019 $test\textrm{-}dev$ set with $180k$ iterations and compare the results to state-of-the-art one-stage and two-stage detectors. The experimental results are presented in Table \ref{Tab.4}. For fair comparison,  we only list the results of competitors of a single model with no sophisticate data argumentation in the training and testing. Without bells and whistle, our method with ResNeX-64x4d-101+FPN achieves $44.1\%$ mAP,  which surpasses the counterparts in the Table \ref{Tab.4} by a large margin.  Compared to the closest competitor Grid R-CNN \cite{GridRCNN2019}, the superiority of the proposed approach is more substantial at the higher IoU threshold (0.75), improving more than $1.0\%$ ($47.6\%$ vs $46.6 \%$), which is consistent with that our method can predict more precise bounding boxes.

\section{Conclusion and Discussion}

Smooth-$\ell_1$ Loss and its variants  dominate the localization loss in modern CNN based detectors. Nevertheless, their oversimplified assumption that four coordinate variables of a bounding box are independent does not accord with the fact. Therefore the localization performance of these detectors might suffer degradation. In light of this, we propose a generalized $EIoU$ to address this problem. To make the $EIoU$ based loss not oscillated in the neighbourhood of  the minimum and steadily optimized in train, we introduce  CT and  SOT. Moreover, we present  IoU Head to further improve localization accuracy.

Very Recently, a wide variety of  anchor-free detectors \cite{CornerNet_2018,FSAF2019,FCOS2019,zhou2019objects,CenterNet2019} were developed and receive more and more attention. We think the proposed $EIoU$ Loss may be more applicable to these detection models, because there may exist more non-overlapping box pairs due to no anchors.

We provide a new route to design $IoU$ based losses, and all the decreasing functions of $IoU$ can be modified and become an applicable localization loss through CT.  We just tried the simplest $-IoU$, and many other functions not limited to $\frac{1}{IoU}$, $-\ln(IoU)$ might be more appropriate.  Therefore there is great potential to further the performance by  exploiting these techniques.

More importantly,  CT and SOT  are so general that they can beyond the field of detection.  CT can help any loss to have zero-gradient at the minimum and make it possible to achieve the minimum through gradient descend algorithms.  SOT can help many types of losses, including but not limited fractional losses ( factional losses are common in machine learning tasks, since we usually need to minimize an objective function and maximize another simultaneously ), to steadily and smoothly arrive at the minimum. Therefore,  CT and  SOT may find more applications in other fields.

{
\bibliographystyle{IEEEtran}
\bibliography{refs}
}

\section*{Appendix}
\subsection*{A. Gradients of Smooth-EIoU Loss  }
In this section, we will deduce all the partial derivative of Smooth-$EIoU$ loss in Eq.(\ref{Eq.22}) w.r.t. $x_1^p$, $y_1^p$, $x_2^p$ and $y_2^p$. We first compute the partial derivative, \emph{i.e.},
\begin{equation}
\frac{\partial I_e}{\partial x_1^p} = \left\{
 \begin{aligned}
    &y_{\min}- y_{\max},  ~~~~~~~~~{\rm if} ~ x_1^p \ge x_1^t ~{\rm and} ~ x_1 \le x_2, \\
    &2y_0 - y_{\max} - y_1,  ~~{\rm if}~  x_1^p \ge x_1^t ~{\rm and}~  x_1 > x_2,  \\
    &0, ~~~~~~~~~~~~~~~~~~~~~~~~~~\;{\rm if}~ x_1^p < x_1^t.
 \end{aligned}
 \right .
 \label{Eq.A_23}
\end{equation}

\begin{equation}
\frac{\partial I_e}{\partial y_1^p} = \left\{
 \begin{aligned}
    &x_{\min}- x_{\max},  ~~~~~~~~~{\rm if} ~ y_1^p \ge y_1^t ~{\rm and} ~ y_1 \le y_2, \\
    &2x_0 - x_{\max} - x_1,  ~~{\rm if}~  y_1^p \ge y_1^t ~{\rm and}~  y_1 > y_2,  \\
    &0, ~~~~~~~~~~~~~~~~~~~~~~~~~~~\;{\rm if}~ y_1^p < y_1^t.
 \end{aligned}
 \right .
 \label{Eq.A_24}
\end{equation}

\begin{equation}
\frac{\partial I_e}{\partial x_2^p} = \left\{
 \begin{aligned}
    &y_2- y_1,  ~~~~~~~~~~~~~~~~~~{\rm if} ~ x_2^p \le x_2^t ~{\rm and} ~ x_1 \le x_2, \\
    &y_2 + y_{\min} - 2y_0,  ~~~{\rm if}~  x_2^p \le x_2^t ~{\rm and}~  x_1 > x_2,  \\
    &0, ~~~~~~~~~~~~~~~~~~~~~~~~~~~\;{\rm if}~ x_2^p > x_2^t.
 \end{aligned}
 \right .
 \label{Eq.A_25}
\end{equation}

\begin{equation}
\frac{\partial I_e}{\partial y_2^p} = \left\{
 \begin{aligned}
    &x_2- x_1,  ~~~~~~~~~~~~~~~~~{\rm if} ~ y_2^p \le y_2^t ~{\rm and} ~ y_1 \le y_2, \\
    &x_2 + x_{\min} - 2x_0,  ~~{\rm if}~  y_2^p \le y_2^t ~{\rm and}~  y_1 > y_2,  \\
    &0, ~~~~~~~~~~~~~~~~~~~~~~~~~~~\;{\rm if}~ y_2^p > y_2^t.
 \end{aligned}
 \right .
 \label{Eq.A_26}
\end{equation}

And then we compute the the  partial derivative of $U_e$
\begin{equation}
\frac{\partial U_e}{\partial x_1^p} = (y_1^p - y_2^p) - \frac{\partial I_e}{\partial x_1^p},
\end{equation}

\begin{equation}
\frac{\partial U_e}{\partial y_1^p} = (x_1^p - x_2^p) - \frac{\partial I_e}{\partial y_1^p},
\end{equation}

\begin{equation}
\frac{\partial U_e}{\partial x_2^p} = (y_2^p - y_1^p) - \frac{\partial I_e}{\partial x_2^p},
\end{equation}

\begin{equation}
\frac{\partial U_e}{\partial y_2^p} = (x_2^p - x_1^p) - \frac{\partial I_e}{\partial y_2^p}.
\end{equation}

Finally we obtain the the partial derivative of Smooth-$EIoU$ Loss w.r.t. $x_1^p$
\begin{align}
  &\frac{{\partial \cal} L_{\rm{Smooth}\textrm{-}\rm{EIoU}}}{\partial x_1^p} = 2\left(1- \frac{I_e}{U_e}\right)\frac{I_e\frac{\partial U_e}{\partial x_1^p} - \frac{\partial I_e}{\partial x_1^p}U_e  }{U_e^2}, \\
  &\frac{{\partial \cal} L_{\rm{Smooth}\textrm{-}\rm{EIoU}}}{\partial y_1^p} = 2\left(1- \frac{I_e}{U_e}\right)\frac{I_e\frac{\partial U_e}{\partial y_1^p} - \frac{\partial I_e}{\partial y_1^p}U_e  }{U_e^2}, \\
  &\frac{{\partial \cal} L_{\rm{Smooth}\textrm{-}\rm{EIoU}}}{\partial x_2^p} = 2\left(1- \frac{I_e}{U_e}\right)\frac{I_e\frac{\partial U_e}{\partial x_2^p} - \frac{\partial I_e}{\partial x_2^p}U_e  }{U_e^2}, \\
  &\frac{{\partial \cal} L_{\rm{Smooth}\textrm{-}\rm{EIoU}}}{\partial y_2^p} = 2\left(1- \frac{I_e}{U_e}\right)\frac{I_e\frac{\partial U_e}{\partial y_2^p} - \frac{\partial I_e}{\partial y_2^p}U_g  }{U_g^2}.
\end{align}

\subsection*{B. Proof of Theorem 1}
\begin{theorem}
If the gradient of $f(x)$, denoted as $\nabla f(x)$, is Lipschitz continuous, \emph{i.e.},
\begin{equation}
\Vert \nabla f(x_1) - \nabla f(x_2)\Vert \le L \Vert x_1 - x_2\Vert_2,
\label{Eq.A_27}
\end{equation}
 the function $g(x)$ is positive and bounded, i.e., $0 < g(x) \le M$, and the learning rate satisfies $\alpha < \frac{1}{LM}$,  the  update update rule,
\begin{equation}
x_{k+1} = x_k - \alpha g(x_k) \nabla f(x_k),
\label{Eq.A_28}
\end{equation}
will make $f(x)$ steadily decrease.
\end{theorem}

\begin{proof}
From Eq. (\ref{Eq.A_27}), we can deduce that
\begin{equation}
f(x) \le f(x_k) + \langle\nabla f(x_k), x - x_k \rangle + \frac{L}{2}\Vert x - x_k \Vert_2^2.
\end{equation}
It is known $\alpha < \frac{1}{LM}$ and $0 < g(x_k) \le M$, hence we have
\begin{equation}
\begin{aligned}
f(x) &\le f(x_k) + \langle\nabla f(x_k), x - x_k\rangle + \frac{1}{2\alpha M}\Vert x - x_k \Vert_2^2 \\
&\le f(x_k) + \langle\nabla f(x_k), x - x_k\rangle + \frac{1}{2\alpha g(x_k)}\Vert x - x_k \Vert_2^2.
\end{aligned}
\label{Eq.30}
\end{equation}

The right side of Eq. (\ref{Eq.30}) can be further equivalently reformulated as
\begin{equation}
\begin{aligned}
P(x;x_k) = &f(x_k) + \langle\nabla f(x_k), x - x_k\rangle + \frac{1}{2\alpha g(x_k)}\Vert x - x_k \Vert_2^2 \\
= &f(x_k) +  \frac{1}{2\alpha g(x)} \Vert x - (x_k - \alpha g(x_k) \nabla f(x_k) )\Vert_2^2 \\
&- \frac{\alpha g(x_k)}{2}\Vert \nabla f(x_k) \Vert_2^2
\end{aligned}
\label{Eq.31}
\end{equation}

It is easy to know  $x_{k+1} = x_k - \alpha g(x_k) \nabla f(x_k) $ in Eq. (\ref{Eq.A_28}) is the minimal point of Eq. (\ref{Eq.31}) , and then we  obtain
\begin{equation}
f(x_{k+1}) \le P(x_{k+1};x_k ) \le P(x_k;x_k) = f(x_k).
\end{equation}
it inidicates $f(x) $ will decrease monotonicly via the update rule in Eq. (\ref{Eq.28}), and $f(x)$   we arrive the conclusion. $\square$
\end{proof}

\subsection*{C. Experimental Examples}
Figure \ref{Figure.8} shows test examples of the VOC2007\_test\_set trained using Smooth-$\ell_1$ Loss, $GIoU$ Loss and the proposed Smooth-$EIoU$ Loss with Faster-RCNN with Res50 backbone and FPN architecture.  The visual results indicate that the Smooth-$IoU$ Loss can help to generate more precise bounding boxes, which  verifies the conclusions in Section 2. Additionally,  comparing with Smooth-$\ell_1$ Loss, $GIoU$ Loss is more helpful to better localization, which is consistent with the results in \cite{GIOU_2019}.

\begin{figure*}[htbp]

    \subfigure{\includegraphics[width=0.8\linewidth]{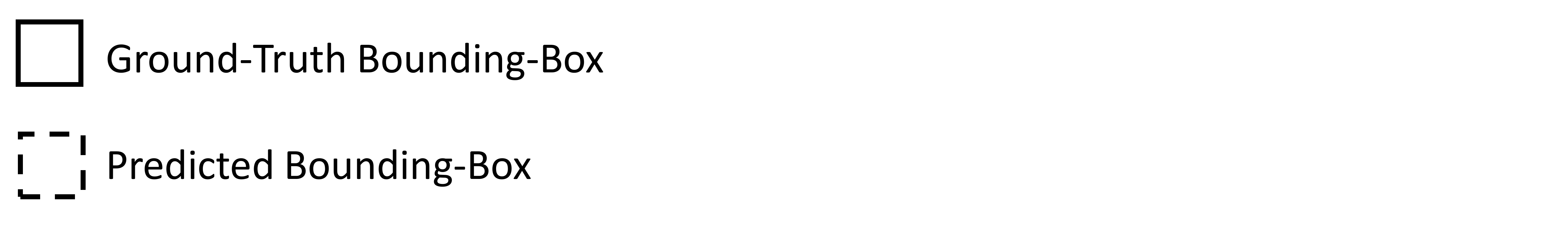}}
    \centering

    \subfigure{\includegraphics[width=0.2\linewidth]{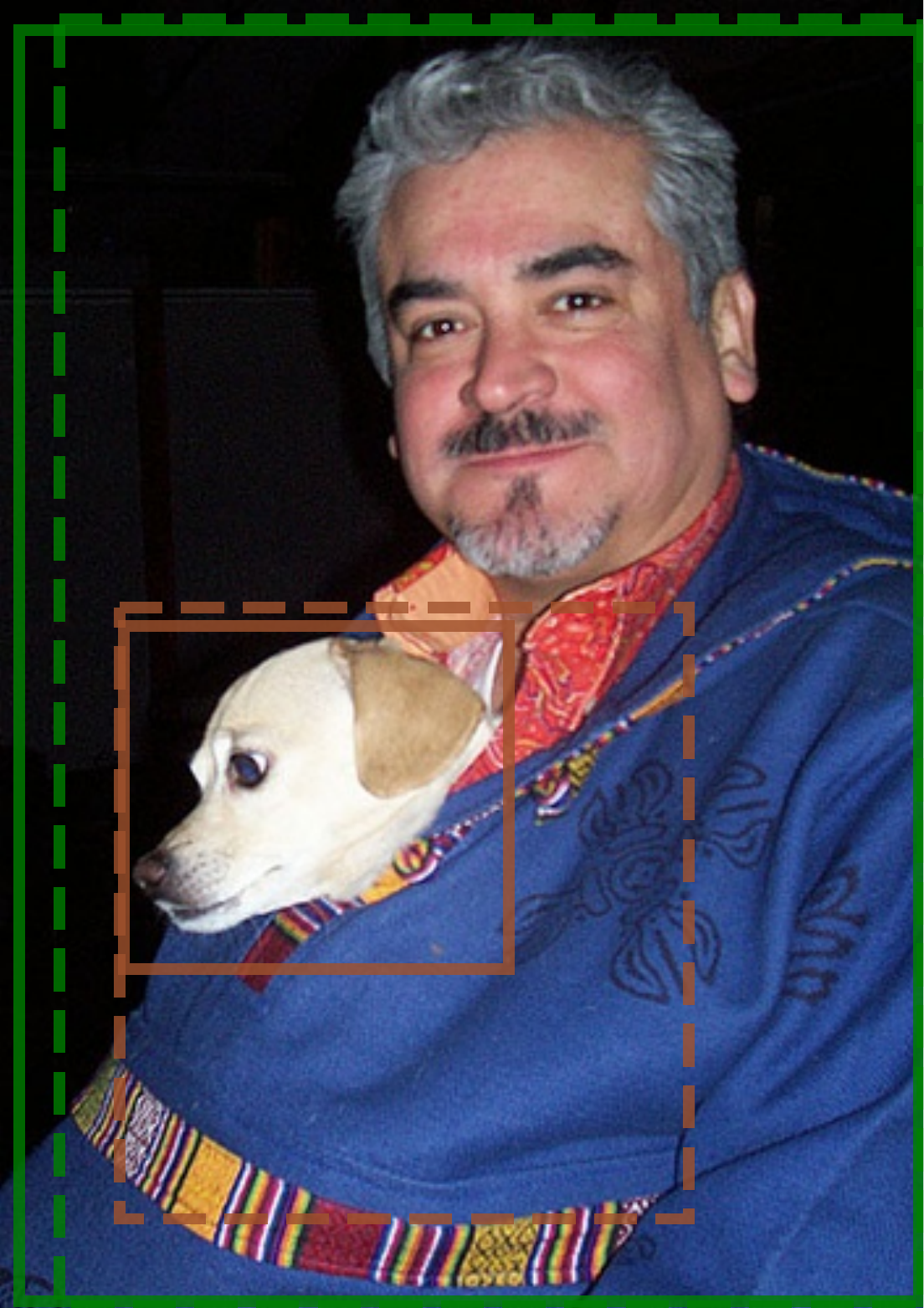}} \hspace{1pt}
    \subfigure{\includegraphics[width=0.2\linewidth]{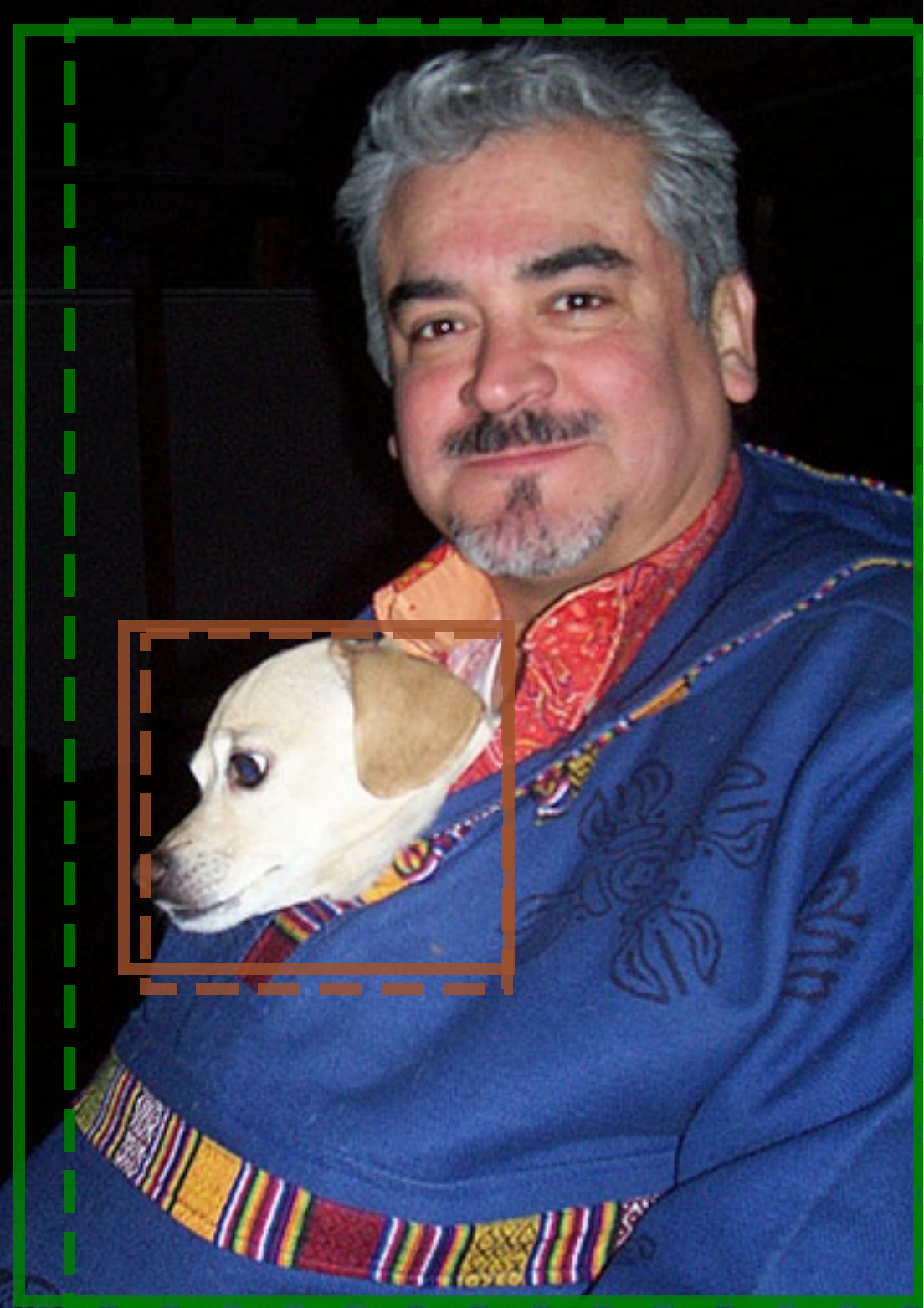}} \hspace{1pt}
    \subfigure{\includegraphics[width=0.2\linewidth]{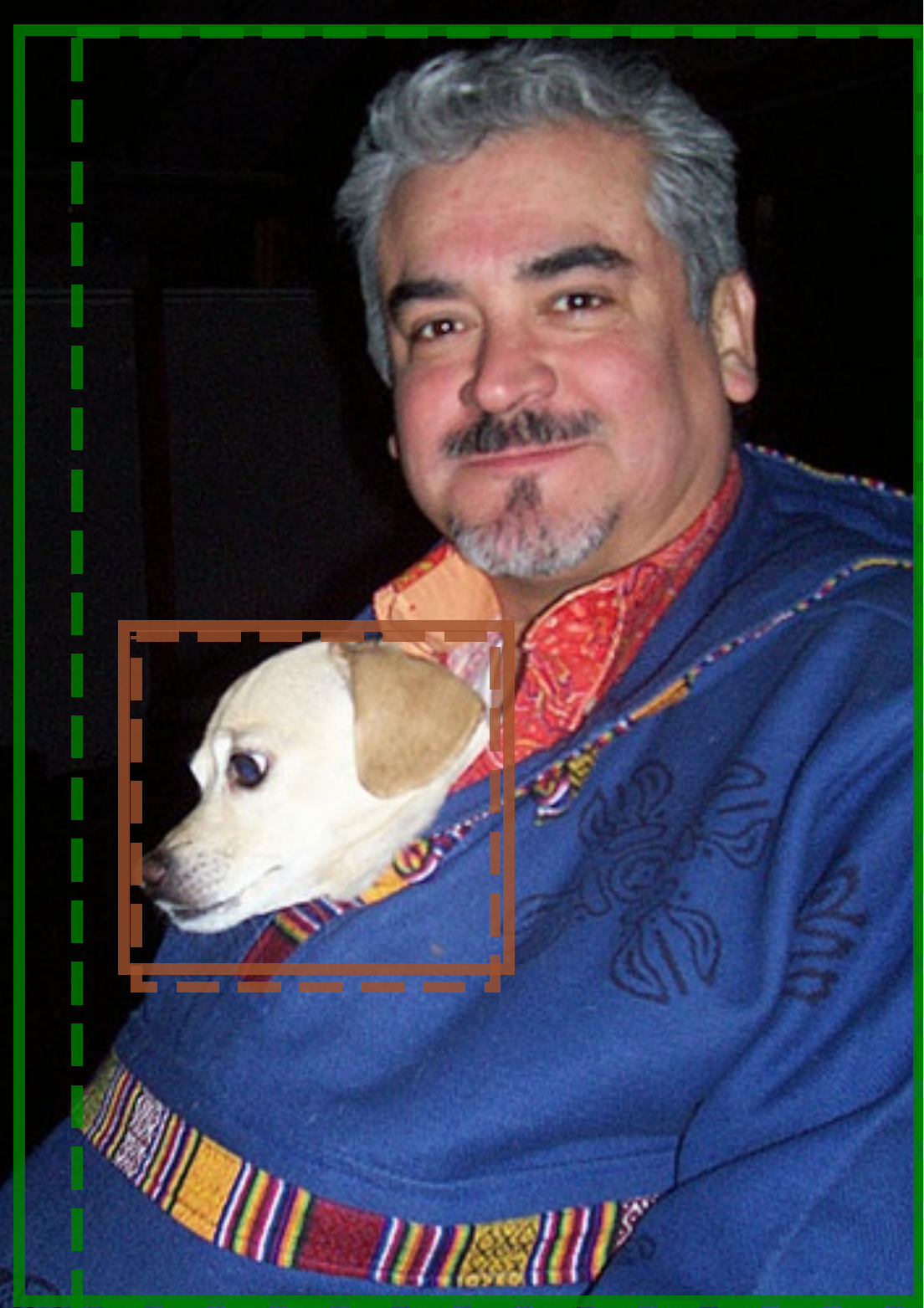}} \hspace{1pt}
    \subfigure{\includegraphics[width=0.2\linewidth]{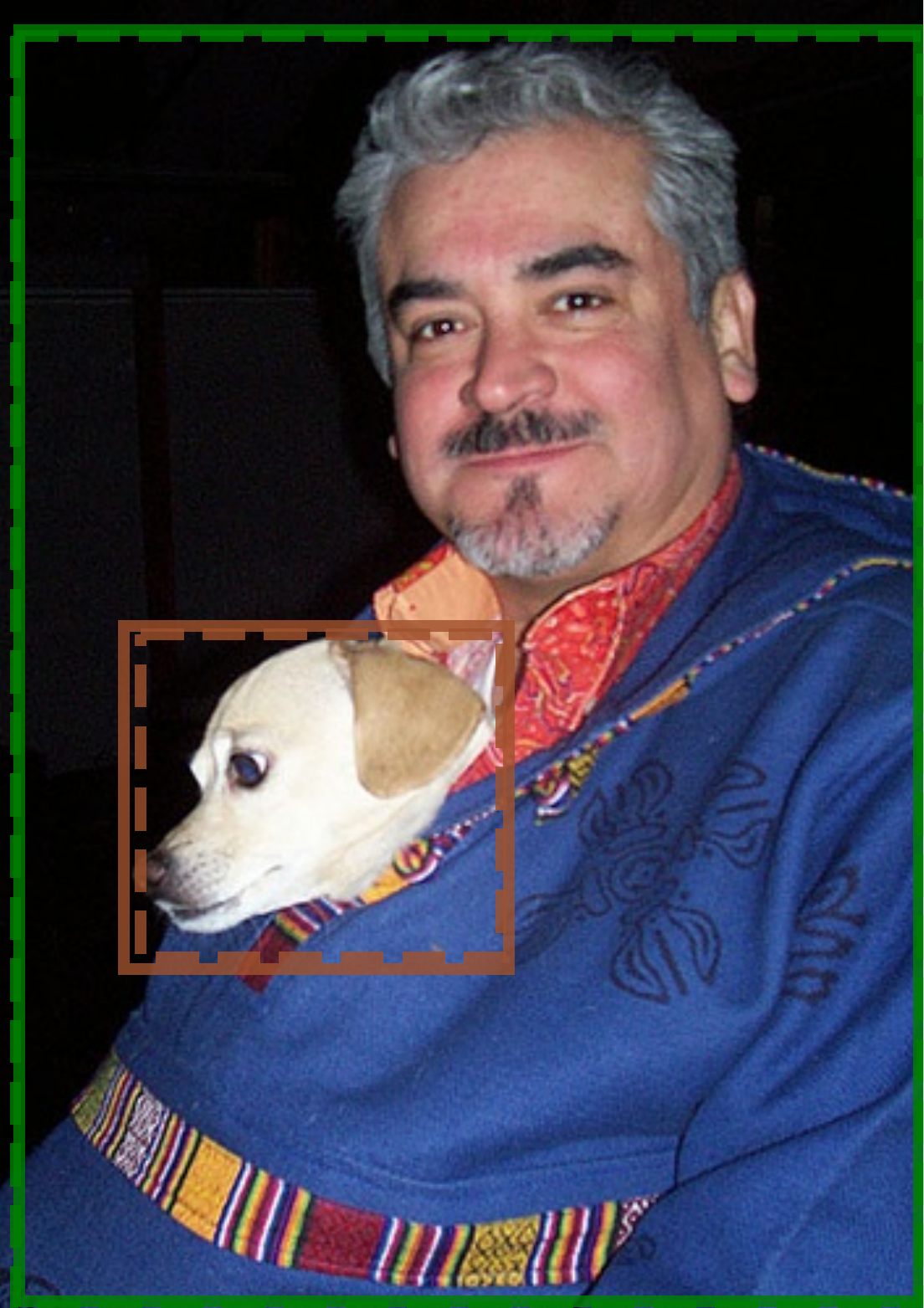}}

    \subfigure{\includegraphics[width=0.2\linewidth]{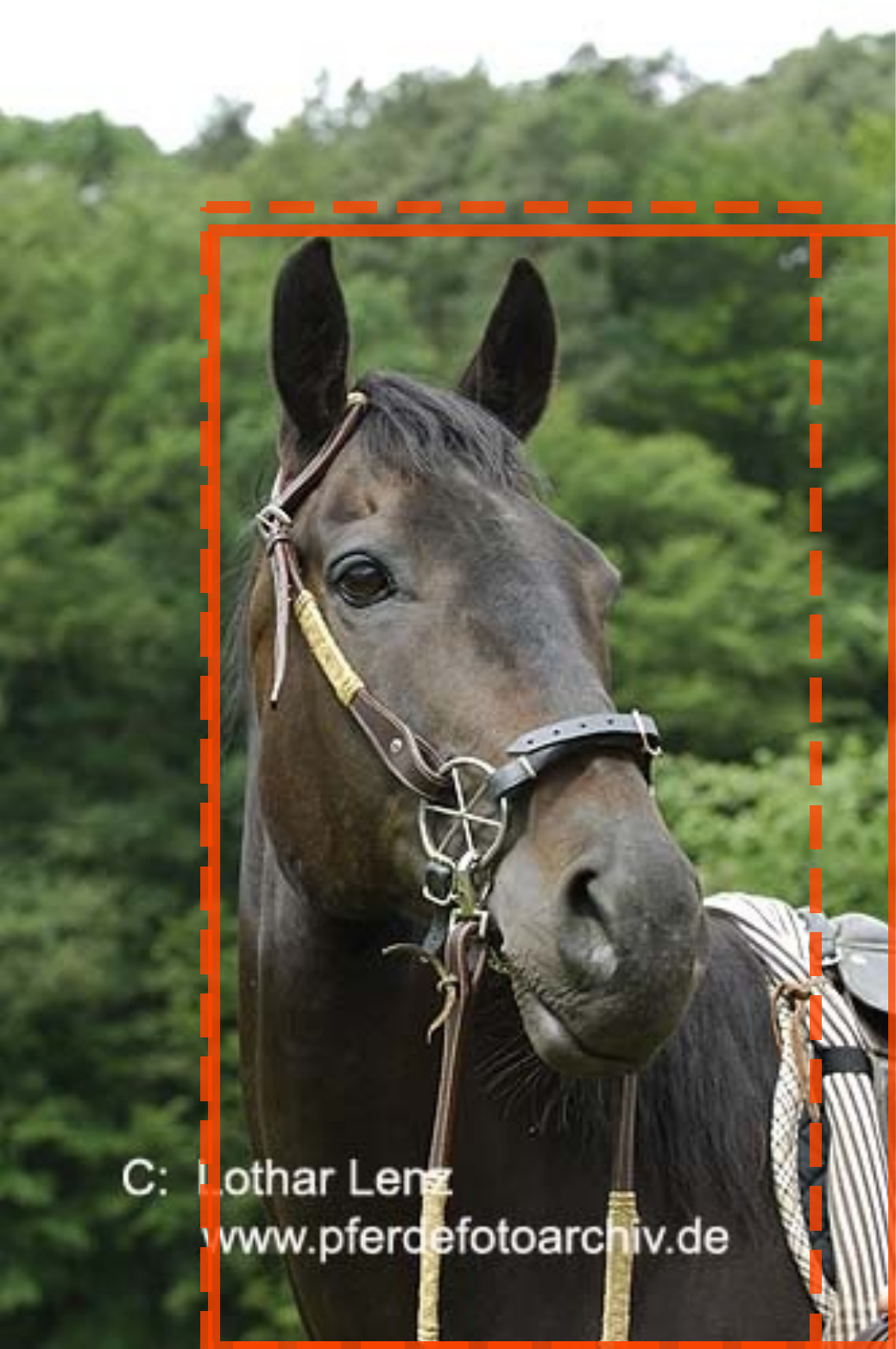}} \hspace{1pt}
    \subfigure{\includegraphics[width=0.2\linewidth]{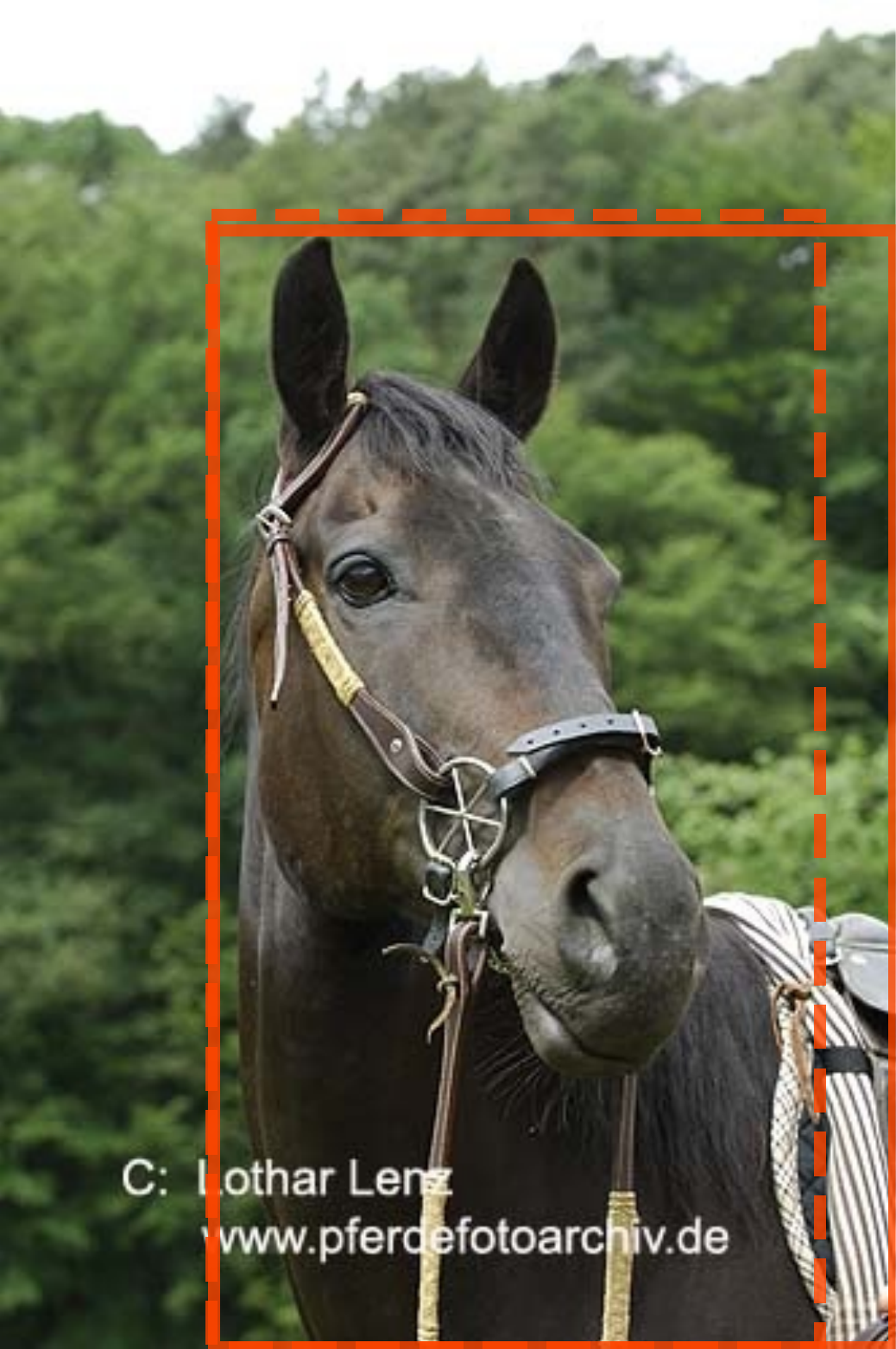}} \hspace{1pt}
    \subfigure{\includegraphics[width=0.2\linewidth]{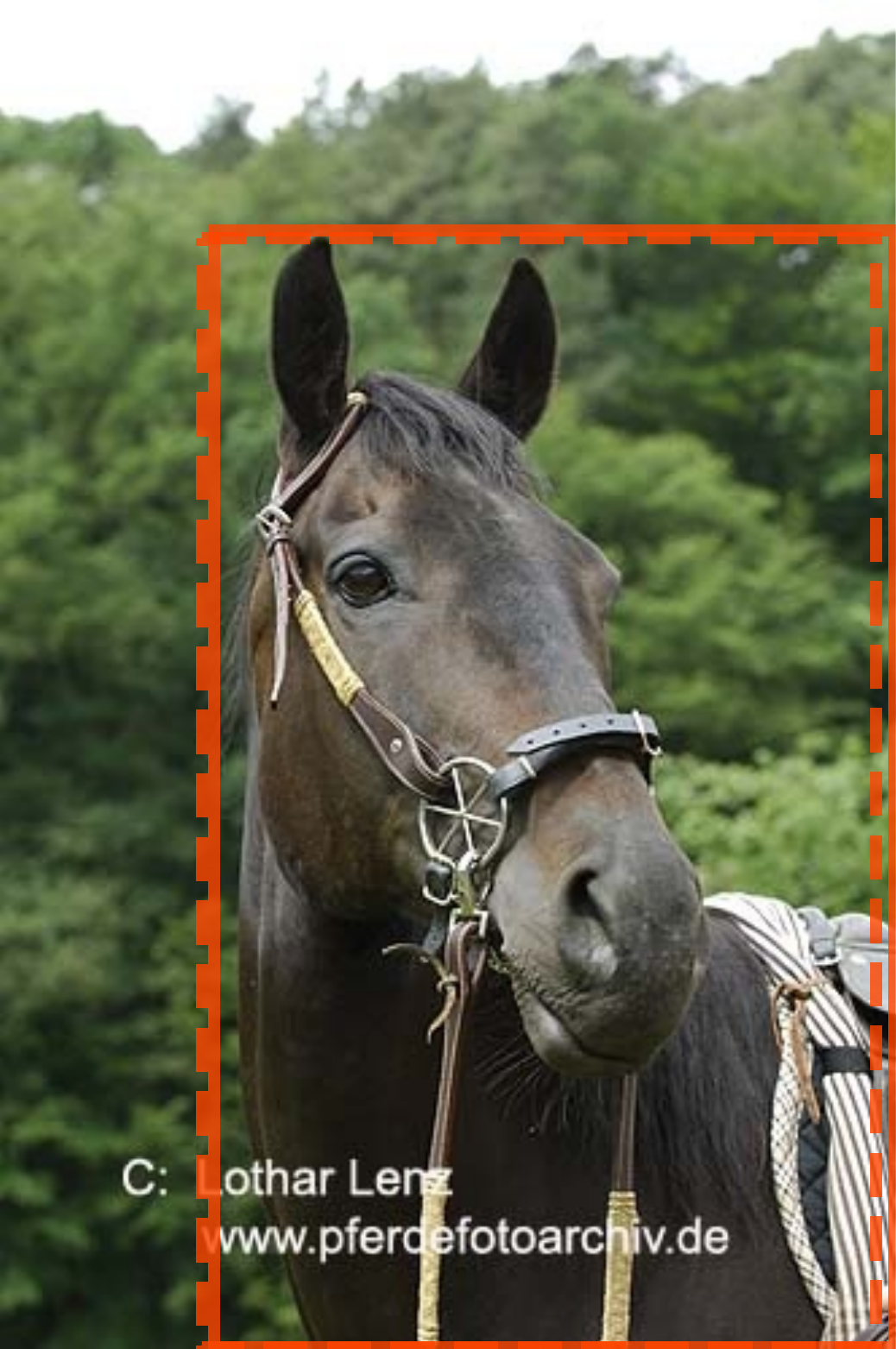}} \hspace{1pt}
    \subfigure{\includegraphics[width=0.2\linewidth]{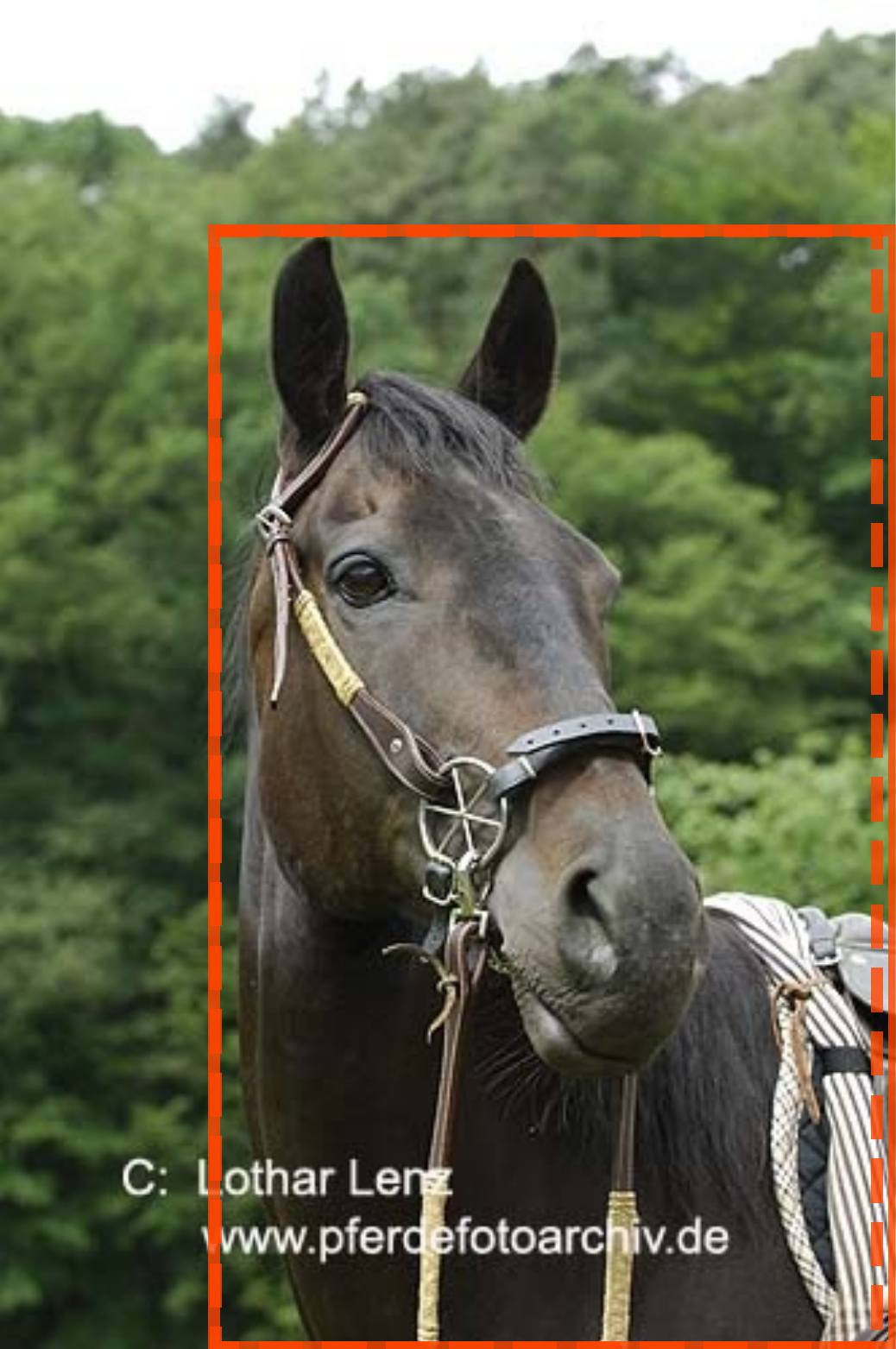}}
    \subfigure{\includegraphics[width=0.2\linewidth]{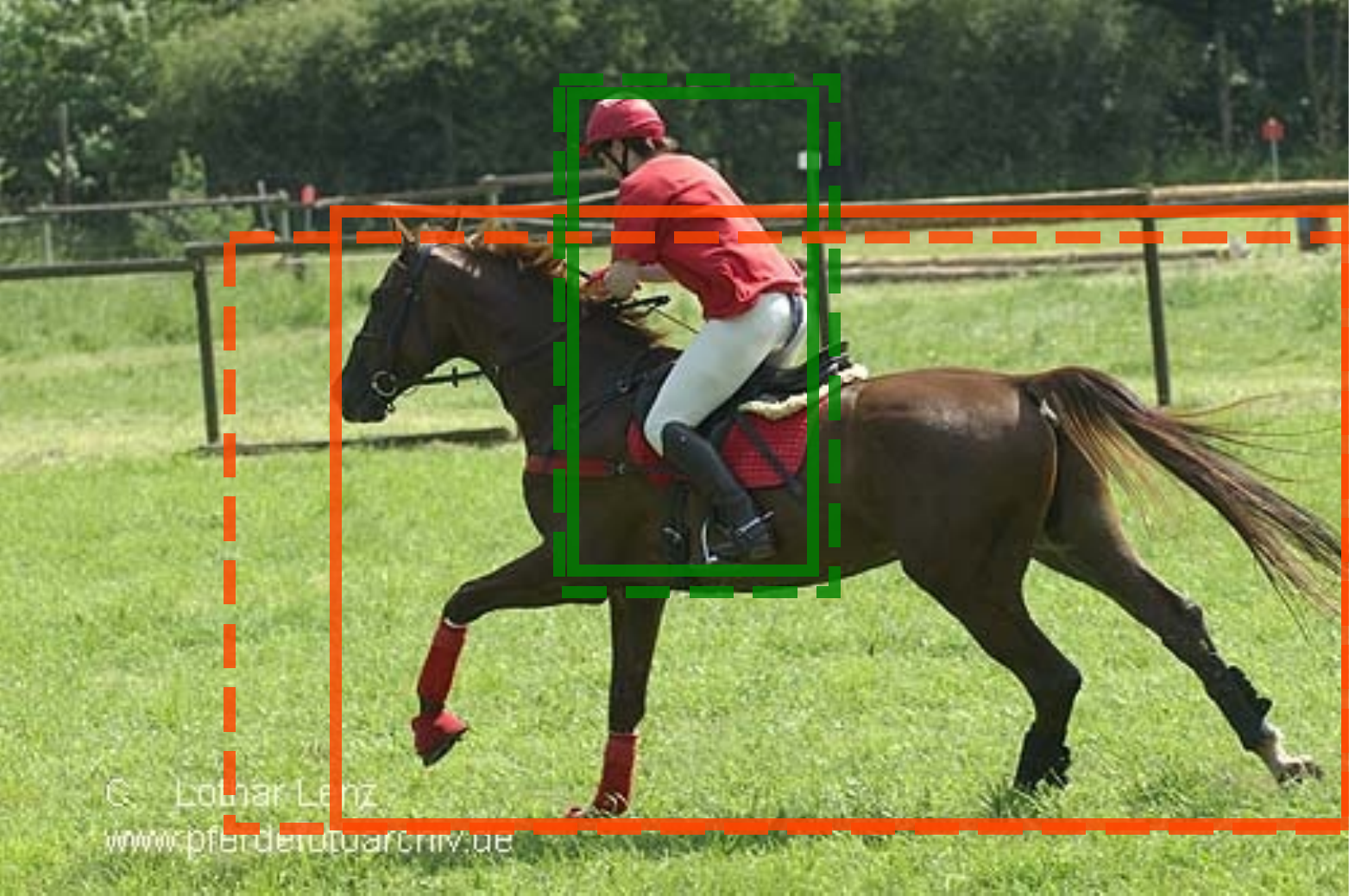}} \hspace{1pt}
    \subfigure{\includegraphics[width=0.2\linewidth]{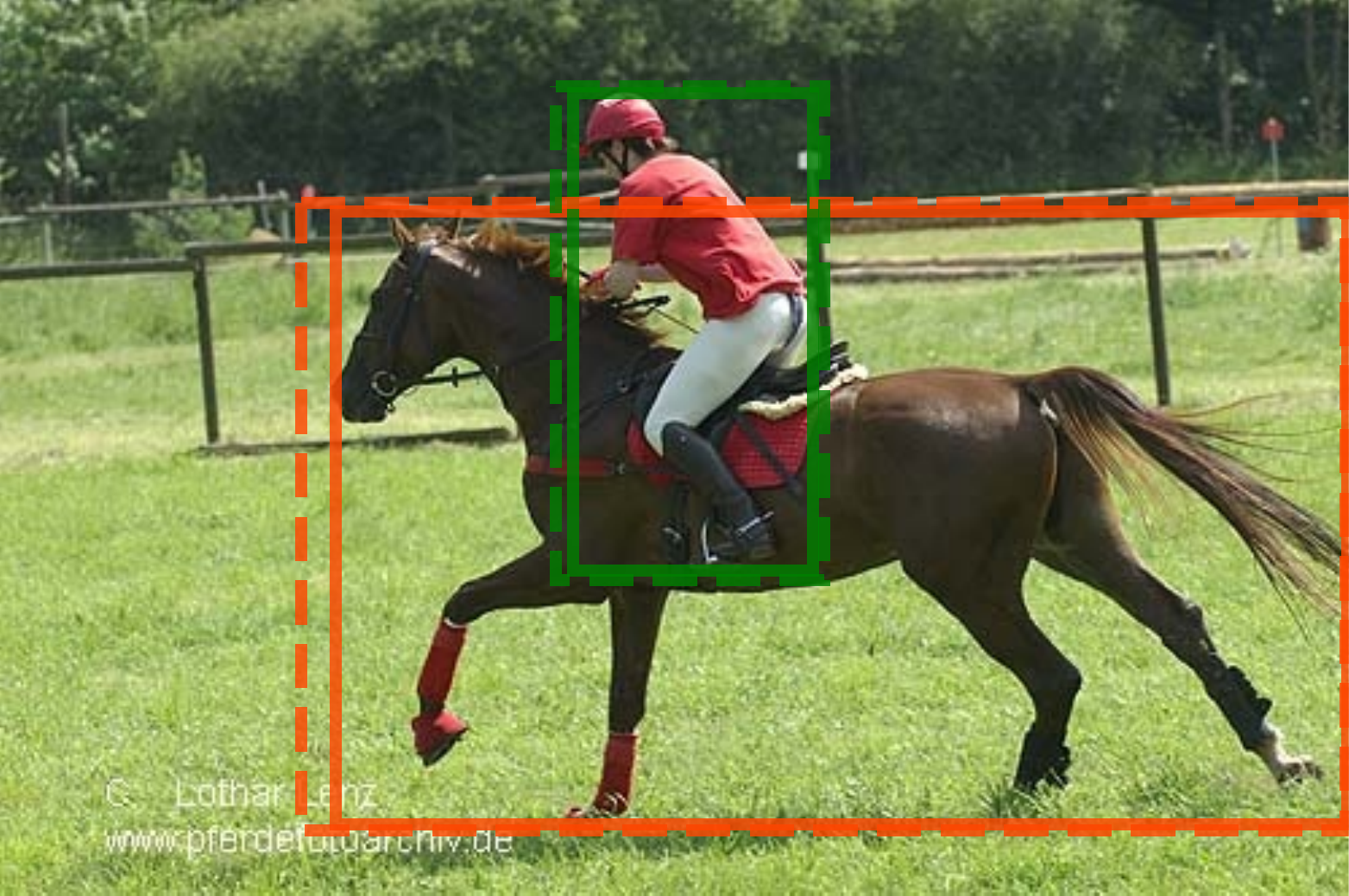}} \hspace{1pt}
    \subfigure{\includegraphics[width=0.2\linewidth]{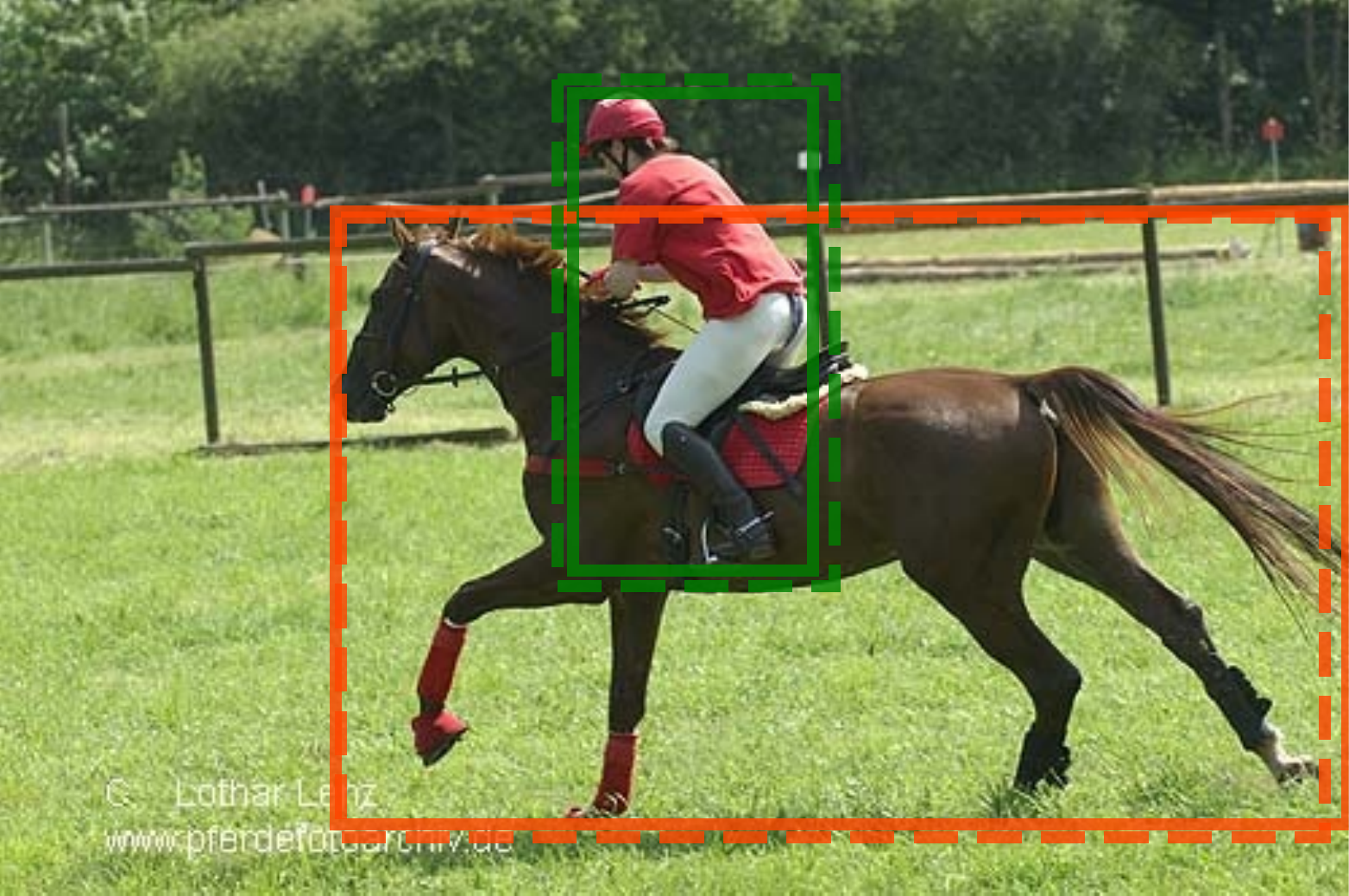}} \hspace{1pt}
    \subfigure{\includegraphics[width=0.2\linewidth]{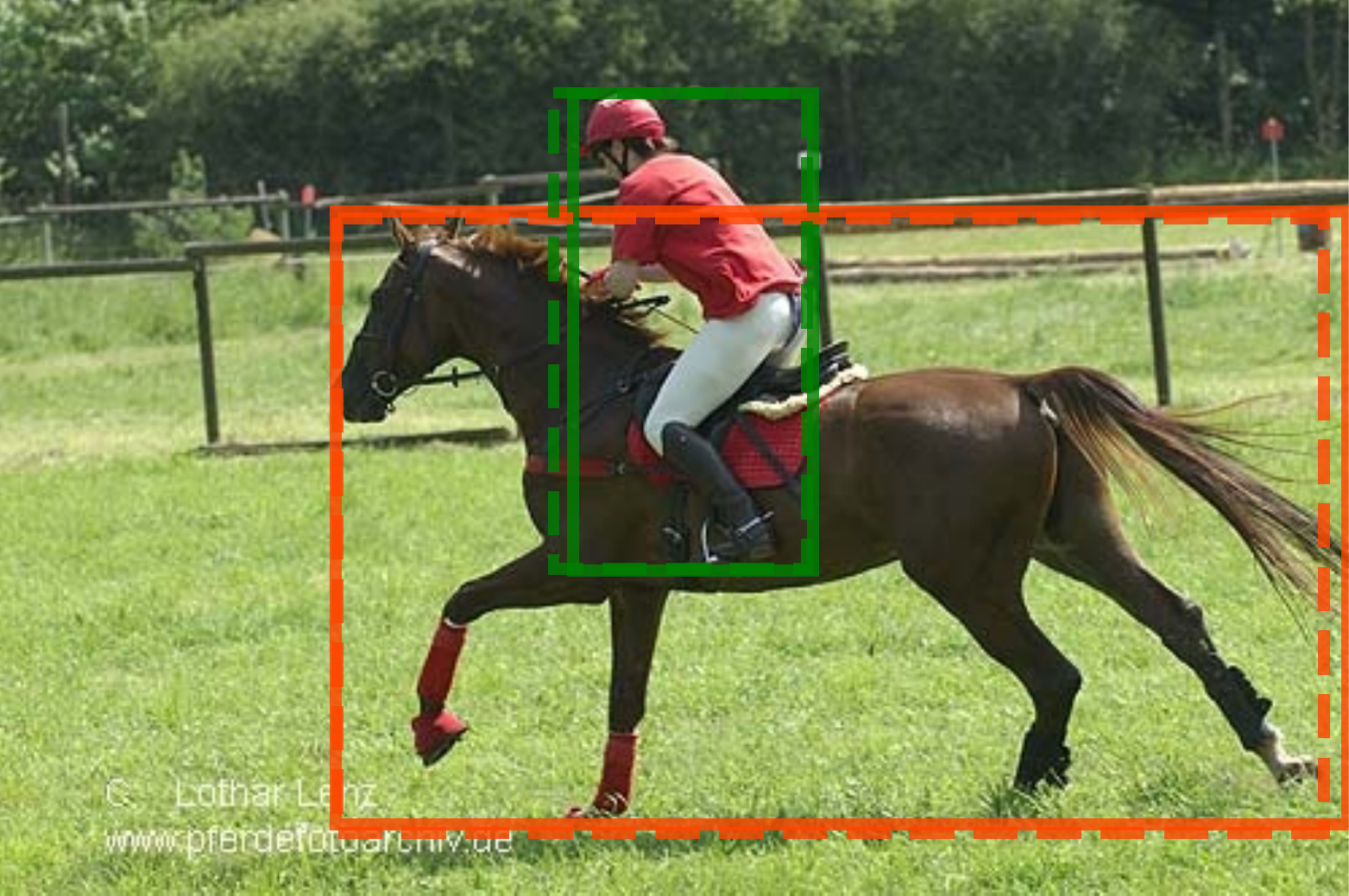}} \\

     \centering
      \hspace{1pt}
    \begin{minipage}{0.2\linewidth}
    \centerline{\includegraphics[width=1\linewidth]{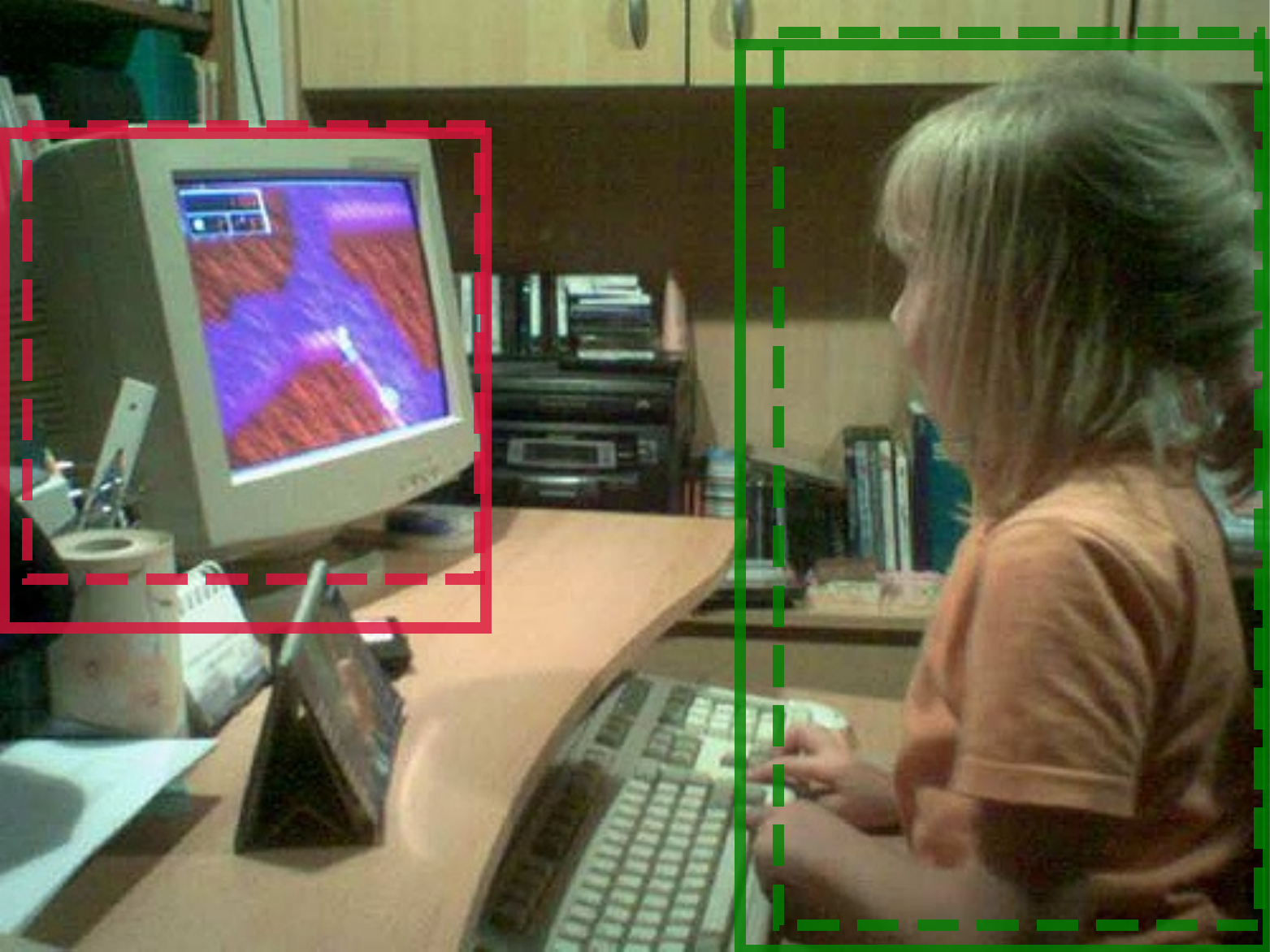}}
    \centerline{Smooth-$\ell_1$ Loss}
    \end{minipage}
    \hspace{1pt}
    \begin{minipage}{0.2\linewidth}
    \centerline{\includegraphics[width=1\linewidth]{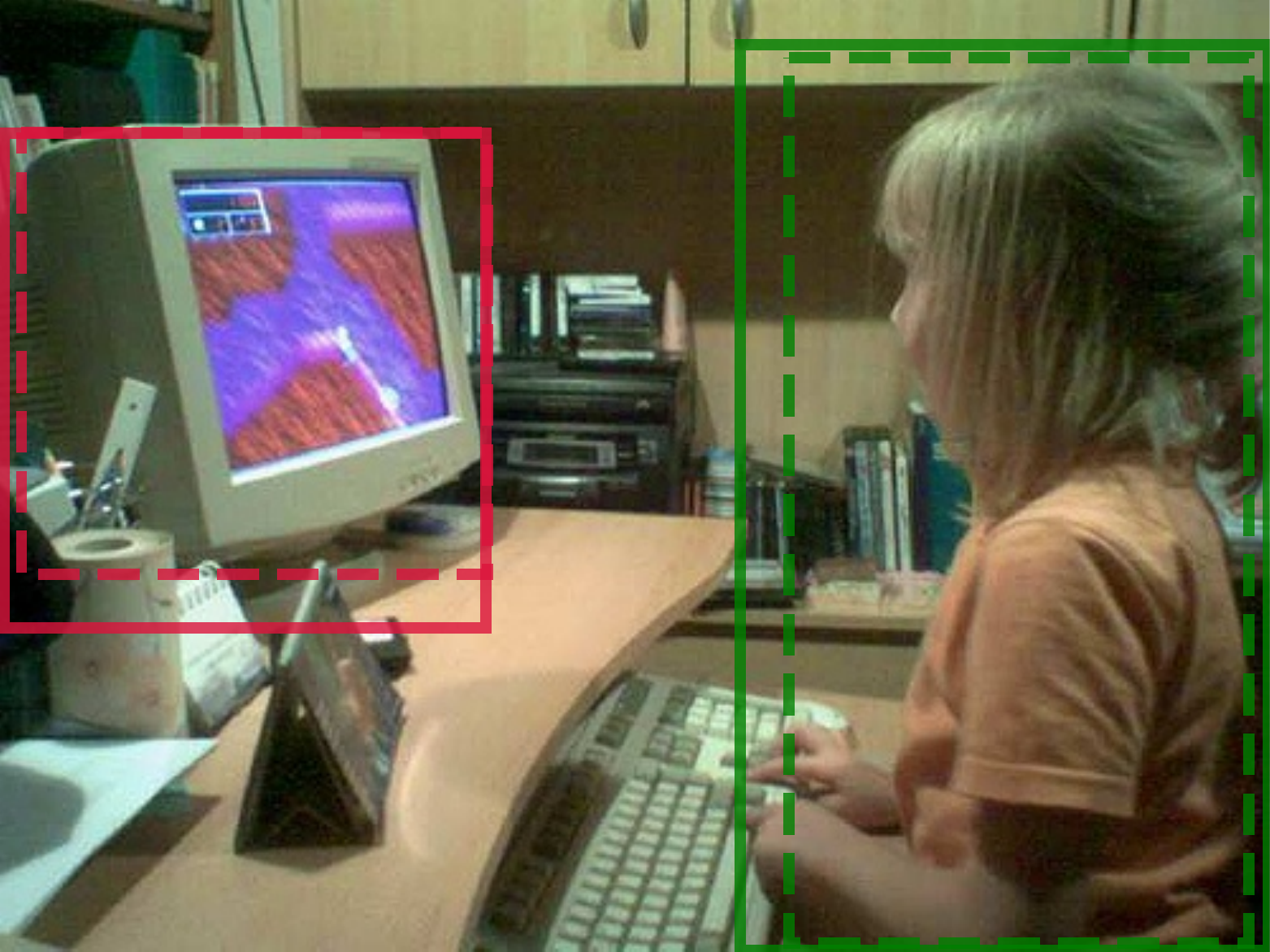}}
    \centerline{$GIoU$ Loss}
    \end{minipage}
    \hspace{1pt}
    \begin{minipage}{0.2\linewidth}
    \centerline{\includegraphics[width=1\linewidth]{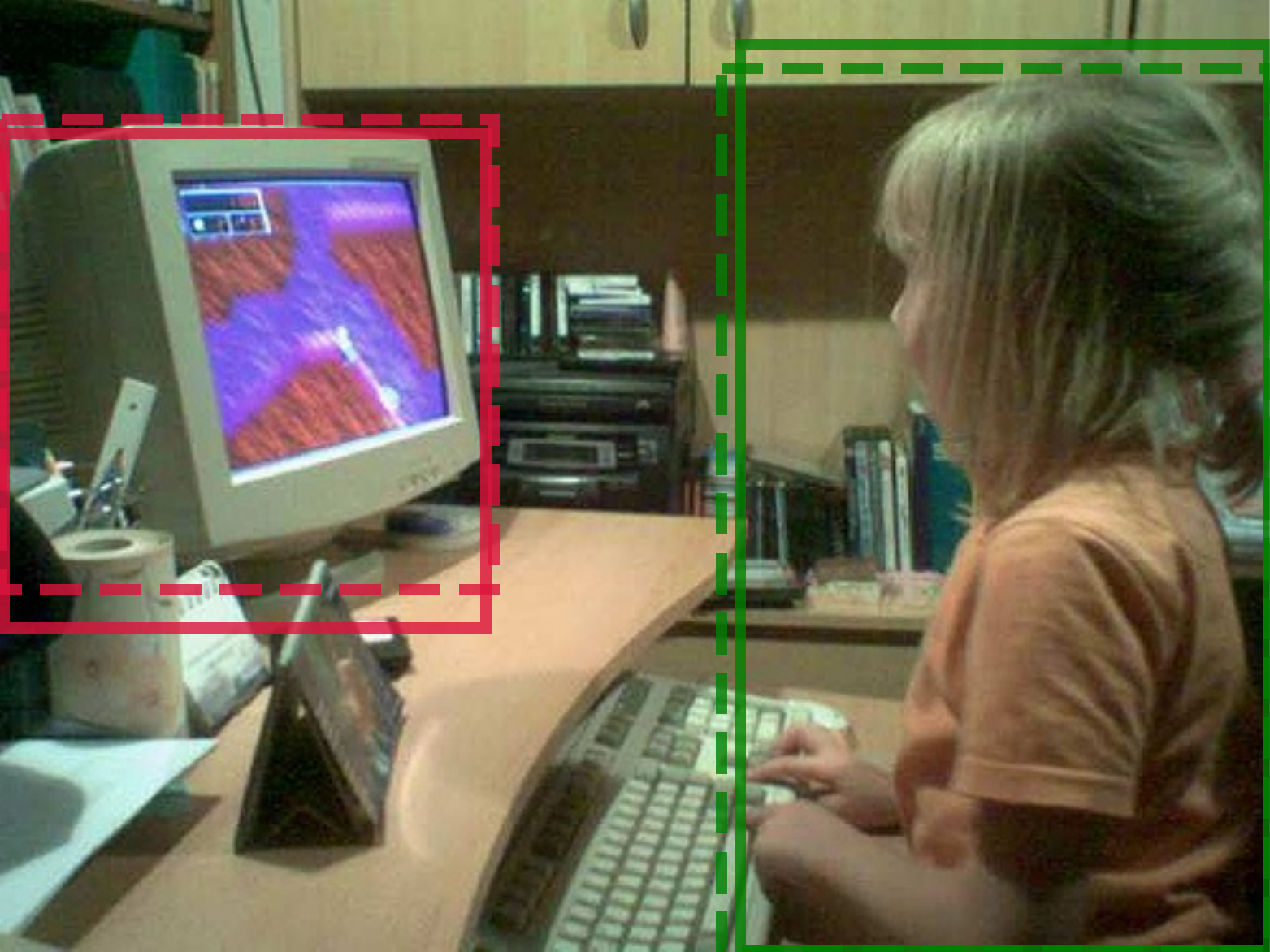}}
    \centerline{$CIoU$ Loss}
    \end{minipage}
    \hspace{1pt}
    \begin{minipage}{0.2\linewidth}
    \centerline{\includegraphics[width=1\linewidth]{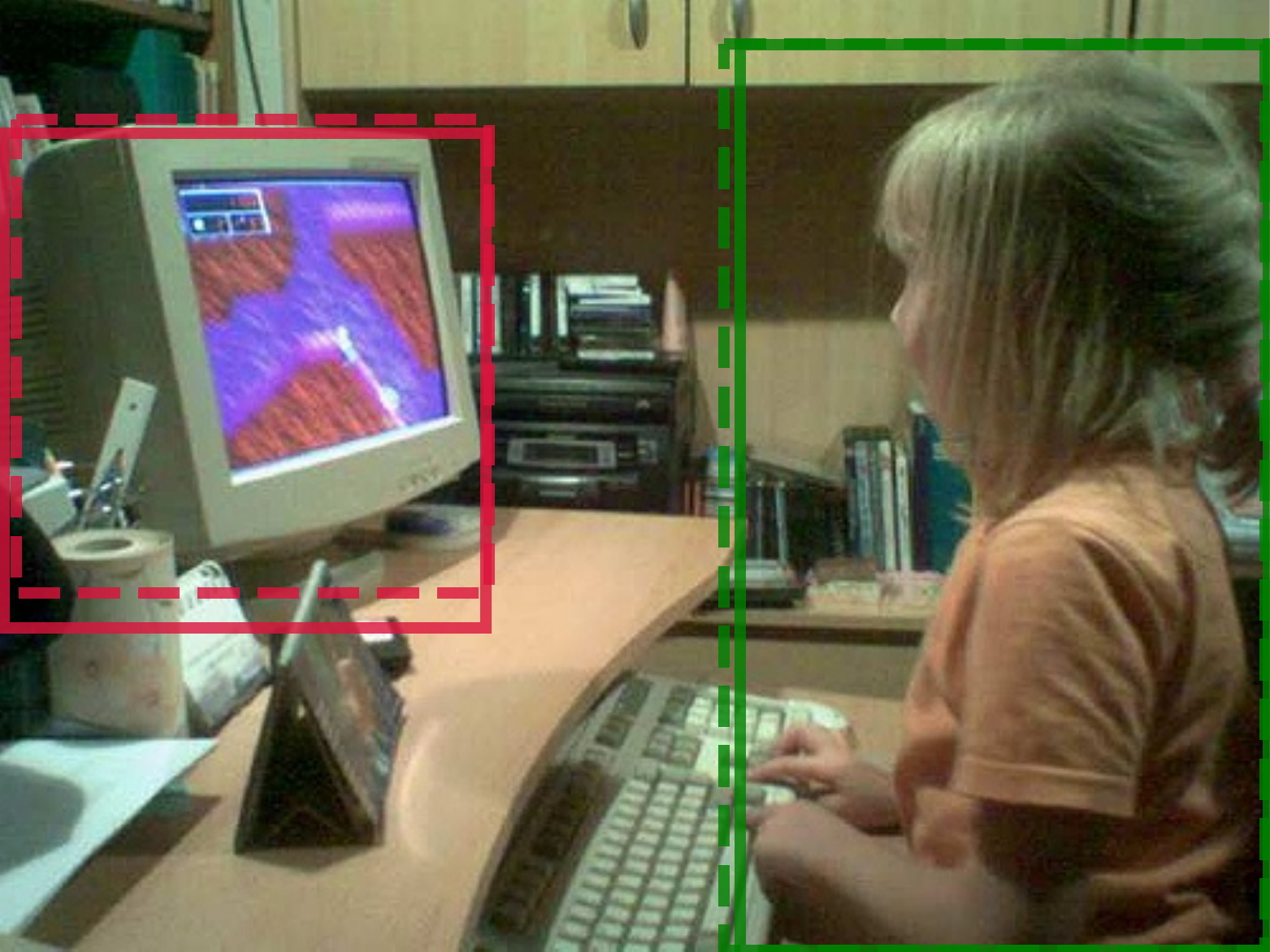}}
    \centerline{Smooth-$EIoU$ Loss}
    \end{minipage}
    \hspace{1pt}

    \caption{ Some test examples of the VOC2007\_test\_set  with  Faster-RCNN with Res50 backbone and FPN architecture trained using  Smooth-$\ell_1$ Loss, $GIoU$ Loss and Smooth-$EIoU$ Loss (\emph{left to right}). Ground-truth boxes are shown with solid lines and the predicted boxes are displayed with dashed lines. }
    \label{Figure.8}
\end{figure*}

\begin{IEEEbiography}[{\includegraphics[height=30mm,clip,keepaspectratio]{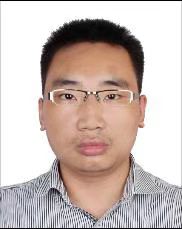}}]{Hanyang Peng} received a B.S. degree in measurement and control technology from the Northeast University of China, Shenyang, China, in 2008, an M.E. degree in detection technology and automatic equipment from the Tianjin University of China, Tianjin, China, in 2010, and a Ph.D. degree in pattern recognition and intelligence systems from the Institute of Automation, Chinese Academy of Sciences, Beijing, China, in 2017. He is currently with Southern University of Science and Technology, Shenzhen, China. His current research interests include computer vision, machine learning, deep learning and optimization.
\end{IEEEbiography}
\begin{IEEEbiography}[{\includegraphics[height=30mm,clip,keepaspectratio]{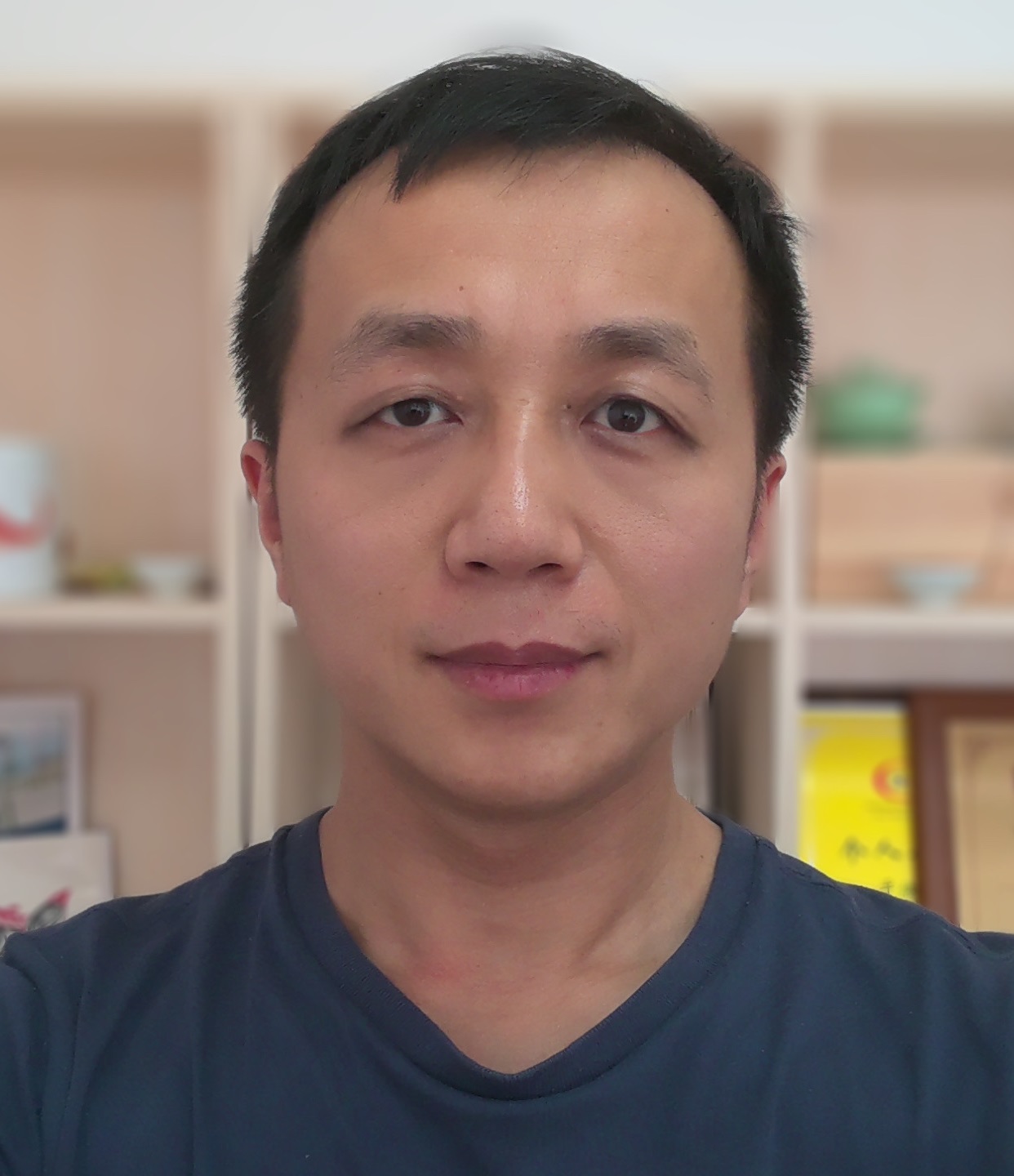}}]{Shiqi Yu} is currently an associate professor in the Department of Computer Science and Engineering, Southern University of Science and Technology, Shenzhen, China. He received his B.E. degree in computer science and engineering from the Chu Kochen Honors College, Zhejiang University in 2002, and Ph.D. degree in pattern recognition and intelligent systems from the Institute of Automation, Chinese Academy of Sciences in 2007. He worked as an assistant professor and an associate professor at Shenzhen Institutes of Advanced Technology, Chinese Academy of Sciences from 2007 to 2010, and as an associate professor at Shenzhen University from 2010 to 2019. His research interests include computer vision, pattern recognition and artificial intelligence.
\end{IEEEbiography}
\end{document}